\documentclass[11pt]{article}

\usepackage[final]{acl}
\usepackage{times}
\usepackage{latexsym}

\usepackage[T1]{fontenc}

\usepackage[utf8]{inputenc}

\usepackage{microtype}

\usepackage{inconsolata}

\usepackage{graphicx}

\usepackage{hyperref}       
\usepackage{url}            
\usepackage{booktabs}       
\usepackage{amsfonts}       
\usepackage{amsmath,amssymb}
\usepackage{nicefrac}       
\usepackage{xcolor}         
\usepackage{soul}           
\usepackage{colortbl}       
\usepackage{algorithm}
\usepackage{algpseudocode}
\usepackage[most]{tcolorbox}

\usepackage{caption}
\usepackage{multirow}

\usepackage{amsthm}

\newtheorem{theorem}{Theorem}[section]
\newtheorem{proposition}[theorem]{Proposition}
\newtheorem{corollary}[theorem]{Corollary}
\newtheorem{lemma}[theorem]{Lemma}
\newtheorem{assumption}[theorem]{Assumption}

\theoremstyle{definition}

\theoremstyle{remark}

\newcommand{\Ours}{SCAS}
\definecolor{revisionblue}{RGB}{222,242,255}
\definecolor{promptFrame}{gray}{0.45}
\definecolor{promptTitle}{gray}{0.90}
\definecolor{promptBack}{gray}{0.985}
\definecolor{promptAnswer}{gray}{0.92}
\sethlcolor{revisionblue}

\soulregister{\citep}{7}
\soulregister{\citet}{7}
\soulregister{\ref}{7}
\soulregister{\emph}{7}
\soulregister{\Ours}{0}

\newcommand{\answerplaceholder}[1]{%
  \fcolorbox{promptFrame}{promptAnswer}{\strut\texttt{#1}}%
}
\newcommand{\promptrole}[1]{\noindent\textbf{#1}}
\newcommand{\promptseparator}{%
  \par\vspace{0.45em}%
  \noindent\textcolor{black!22}{\rule{\linewidth}{0.4pt}}%
  \par\vspace{0.45em}%
}

\newtcolorbox{teacherpromptbox}[1]{%
  enhanced,
  colback=promptBack,
  colframe=promptFrame,
  colbacktitle=promptTitle,
  coltitle=black,
  fonttitle=\bfseries\small,
  fontupper=\small,
  title={#1},
  boxrule=0.9pt,
  arc=0mm,
  left=2.5mm,
  right=2.5mm,
  top=2mm,
  bottom=2mm
}

\newtcolorbox{evalpromptbox}[1]{%
  enhanced,
  colback=promptBack,
  colframe=promptFrame,
  colbacktitle=promptTitle,
  coltitle=black,
  fonttitle=\bfseries\small,
  fontupper=\small,
  title={#1},
  boxrule=0.9pt,
  arc=0mm,
  left=2.5mm,
  right=2.5mm,
  top=2mm,
  bottom=2mm
}
\usepackage{tikz}

\DeclareRobustCommand{\circlednum}[1]{%
  \tikz[baseline=(C.base)]{
    \node[
      fill=black,
      text=white,
      circle,
      inner sep=0pt,
      minimum size=0.9em,
      font=\tiny\bfseries
    ] (C) {#1};
  }%
}

\title{When the Strongest Teacher Is Not the Best Teacher: Student-Centric Answer Selection}

\author{
\textbf{Zhengyu Hu}$^{1,*}$,
~~\textbf{Zheyuan Xiao}$^{2,*}$,
~~\textbf{Linxin Song}$^{3}$,
~~\textbf{Fengqing Jiang}$^{1}$,
~~\textbf{Yuetai Li}$^{1}$,\\
~~\textbf{Zhihan Xiong}$^{1}$,
~~\textbf{Yue Liu}$^{4}$,
~~\textbf{Junhao Lin}$^{5}$,
~~\textbf{Yao Su}$^{6}$,
~~\textbf{Lijie Hu}$^{7}$,\\
~~\textbf{Kaize Ding}$^{8}$,
~~\textbf{Teng Xiao}$^{1,9}$,
~~\textbf{Radha Poovendran}$^{1}$\\
\normalfont
$^1$ University of Washington
\quad $^2$ University of Texas at Austin
\quad $^3$ University of Southern California\\
\quad $^4$ National University of Singapore
\quad $^5$ Microsoft
\quad $^6$ WPI\\
$^7$ MBZUAI
\quad $^8$ Northwestern University
$^{9}$ Allen Institute for AI (AI2)\\
\small{$^*$Equal contribution.}
}

\begin{document}
\maketitle

\begin{abstract}
LLM training increasingly relies on teacher-generated supervision, from synthetic responses to reasoning traces and tool-use demonstrations.
Current practice often chooses the highest-performing teacher to generate student training data, implicitly treating teacher test performance as a proxy for teaching quality.
We show that this assumption can fail: even when multiple teachers provide correct answers to the same question, the answer from the strongest teacher is not necessarily the best supervision for a given student.
To address this gap, we propose \emph{Student-Centric Answer Sampling} (\Ours{}), a framework that selects from verified teacher-generated answers according to their estimated student-centric learning cost.
Motivated by a token-wise gradient decomposition, we derive an efficient forward-only proxy for this cost and use it to guide answer selection during training.
Experiments across 30 teacher models, 6 student base models, and 6 tasks show that \Ours{} consistently improves student performance, suggesting that effective distillation should prioritize supervision matched to the current student rather than teacher strength alone.
\end{abstract}

\section{Introduction}

Modern LLM training increasingly relies on model-generated supervision, as high-quality human annotation is expensive to collect and limited in coverage~\citep{Hu2024ExplainingLB, Chen2025RevisitingSL, Hu2025PopulationAlignedPG, Djuhera2025FixingII, Nad2025SyntheticDG}. 
Teacher models now produce synthetic responses~\citep{Wang2022SelfInstructAL, Xu2024MagpieAD, Xu2026DS2InstructDD}, reasoning traces~\citep{Yu2025CoTSelfInstructBH, Yang2026WhichRT}, agent trajectories~\citep{Li2025FromWT, Wang2025AdaptingWA, Chen2025TrainingLA}, and tool-use demonstrations~\citep{Liu2024ToolACEWT, Xu2025TOUCANS1, Luo2025SelfTrainingLL} at scale, making teacher-generated data central to distillation-style LLM training~\citep{Fang2025KnowledgeDA, Shen2025CYCLEINSTRUCTFS, Burns2023WeaktoStrongGE, Gu2023MiniLLMOD, Agarwal2023OnPolicyDO}. 
In practice, teacher selection is often teacher-centric: the strongest available model is chosen to generate student training data~\citep{Xu2024MagpieAD, Abdin2024Phi4TR, Xu2025KodCodeAD, Li2025FromWT}, implicitly assuming that \emph{teacher competence} is a reliable proxy for \emph{teaching quality}.

However, our teacher--student comparison as shown in Figure~\ref{fig:intro_image} challenges this assumption.
We evaluate nine teacher models on the test set of each task, then use each teacher to generate training answers for Qwen2.5 students of three sizes across DeepScaleR, GSM8K, and MATH.
To isolate teaching quality from answer correctness, the training data for this comparison is restricted to the largest common subset of questions for which all nine teachers produce correct answers; details are provided in Appendix~\ref{app:intro-figure-detail}.
Despite this controlled setting, higher teacher accuracy does not consistently translate into higher student accuracy.
A teacher may have stronger absolute task performance, yet the supervision it provides need not be equally effective for every student.

\begin{figure}[t]
\captionsetup{skip=2pt}
\centering
\includegraphics[width=1.0\linewidth]{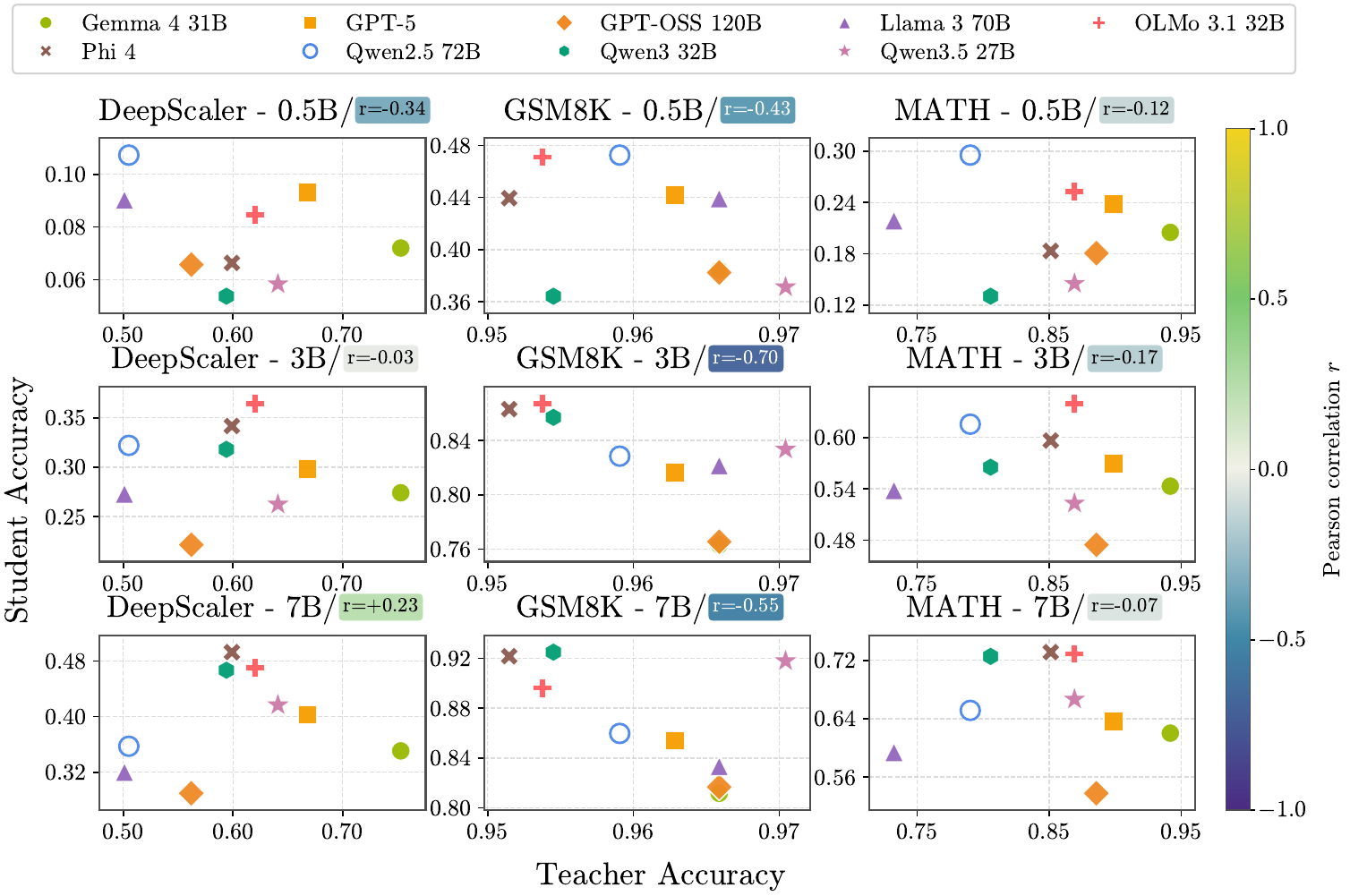}
\caption{\textbf{Teacher--student Accuracy Comparison.}
Higher teacher accuracy does not consistently yield higher student accuracy.
Colored \(r\) values denote the per-panel teacher--student rank correlation.\protect\footnotemark
}
\label{fig:intro_image}
\vspace{-1.0em}
\end{figure}
\footnotetext{See Appendix~\ref{app:intro-figure-detail} for the definition of
the correlation coefficient.}

These observations suggest that supervision quality should be viewed as student-centric.
Given multiple teachers for the same question, the best supervision may not come from the teacher with the highest test performance, but from the answer that best matches the current student.

To make this idea concrete, we view the suitability of a teacher-generated answer through the update it would induce on the current student.
In supervised fine-tuning, pairing the same question with different correct answers creates different token sequences, which lead to different hidden states, prediction errors, and gradients.
We therefore analyze the gradient induced by a question--answer sequence and decompose its squared norm into answer--answer (AA), answer--conditioning (AQ/QA), and conditioning--conditioning (QQ) interaction blocks.
This decomposition clarifies why correctness alone does not determine supervision quality: two correct teacher answers can require different changes from the student, depending on how their answer tokens interact with one another and with the question context.
We refer to this student-dependent update burden as \textbf{student-centric learning cost}.
To avoid expensive backpropagations for computing these gradient interactions, we derive an efficient forward-only proxy for estimating student-centric learning cost.
Then, building on this proxy, we propose \textbf{Student-Centric Answer Sampling} (\Ours{}), a training framework that selects from verified teacher-generated answers according to their estimated student-centric learning cost. Rather than defaulting to the highest-performing teacher, \Ours{} prioritizes lower-cost answers that better match the current student's learning state. To evaluate this idea, we collect candidate answers from 30 widely used open- and closed-source models and conduct experiments across 6 different student base models and 6 representative tasks. Our results show that student-centric answer selection is an effective principle for turning diverse teacher generations into better training data.

\begin{figure*}[t]
\centering
\includegraphics[width=1.0\linewidth]{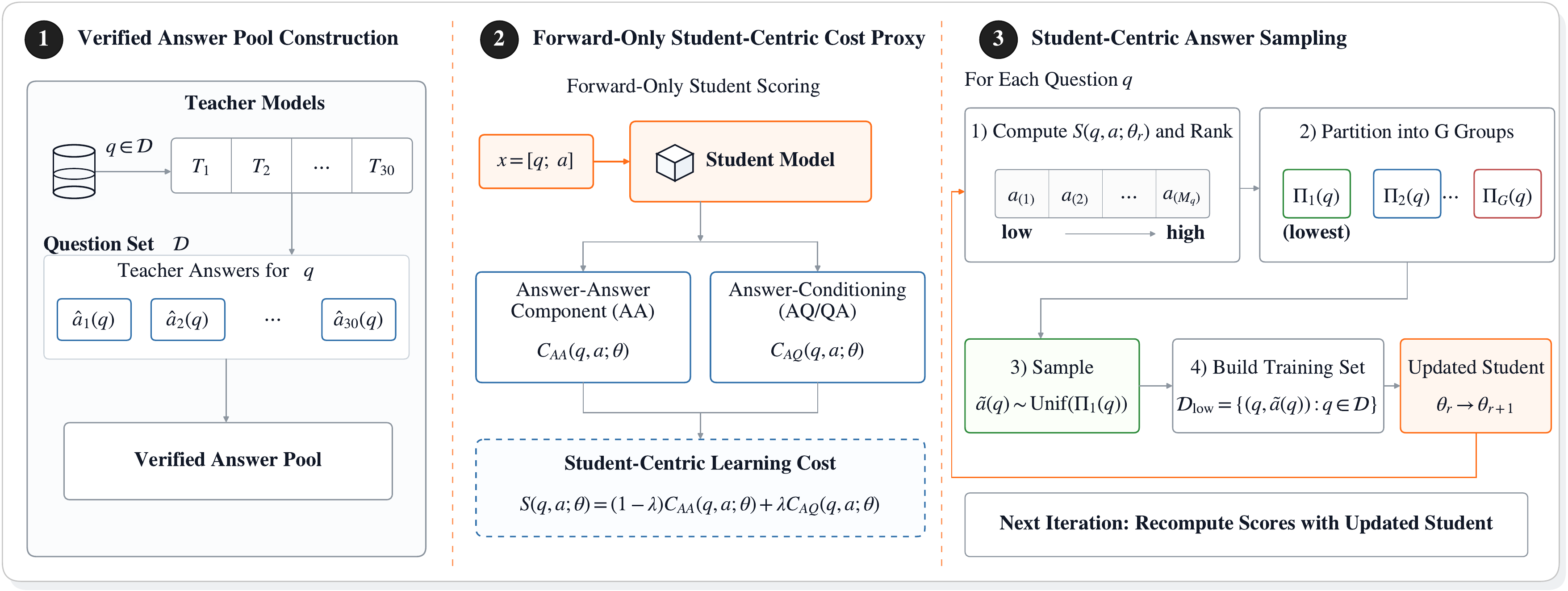}
\caption{
\textbf{Framework overview of Student-Centric Answer Sampling (\Ours{}).}
\circlednum{1} A verified answer pool is constructed from multiple teacher-generated answers for each question.
\circlednum{2} The current student estimates a forward-only student-centric learning cost for each candidate answer.
\circlednum{3} Candidates are ranked and grouped by the estimated cost, and low-cost answers are sampled to form a round-specific training set for updating the student.
Scores are recomputed with the updated student in the next iteration.
}
\label{fig:framework}
\end{figure*}

\section{Method}
\label{sec:method}

\vspace{-1mm}
We organize this section as follows.
Section~\ref{subsec:formulation-motivation} first formalizes the candidate-answer setting and motivates student-centric learning cost, starting from the verified answer pool illustrated in \circlednum{1}.
Section~\ref{subsec:score} then derives an efficient forward-only proxy for estimating this cost under the current student, corresponding to the scoring module in \circlednum{2}.
Section~\ref{subsec:group-training} finally describes how these scores are used to sample training answers and iteratively update the student, as shown in \circlednum{3}.
The overall framework is shown in Figure~\ref{fig:framework}.

\subsection{Formulation and Motivation}
\label{subsec:formulation-motivation}

\paragraph{Setup and notation.}
Let $\mathcal{D}=\{q_i\}_{i=1}^N$ denote the training questions. 
For each $q\in\mathcal{D}$, let $\mathcal{A}(q)$ denote the set of verified correct teacher-produced answers\footnote{See Appendix~\ref{app:candidate-verification} for verification details.} for $q$. The student is an autoregressive language model $p_\theta(\cdot)$ with parameters $\theta$. 
For a candidate answer $a\in\mathcal{A}(q)$, we form the training-templated sequence $x=[q;\,a]$. Let $I=\{1,\dots,|x|\}$ denote the token positions of $x$, and write $I=I_Q\cup I_A$, where $I_Q=\{1,\dots,|q|\}$ indexes the question tokens and $I_A=\{|q|+1,\dots,|x|\}$ indexes the answer tokens. For each token position $t\in I$, let $h_t$ denote the last-layer hidden state at position $t$ under the training template. A summary of notation is provided in Appendix~\ref{subsec:notation}.

\paragraph{Token-wise gradient decomposition.}
We first characterize how token-level interactions contribute to the gradient of a generic student weight matrix. Let $W\in\mathbb{R}^{d\times m}$ be any weight matrix in the student, let $X=[z_t^\top]_{t\in I}\in\mathbb{R}^{|I|\times d}$ denote the token representations feeding into $W$, and let $A=XW\in\mathbb{R}^{|I|\times m}$ be the corresponding pre-activations. For a given layer, the rows of $X$ can be taken as the token representations at that layer feeding into $W$; our derivation is agnostic to the specific choice. Writing $G=[g_t^\top]_{t\in I}\in\mathbb{R}^{|I|\times m}$ for the loss gradient with respect to $A$, standard matrix calculus gives
\begin{equation}
\label{eq:grad}
\nabla_W\mathcal{L} \;=\; X^\top G.
\end{equation}

We are interested in the squared Frobenius norm of this gradient,
\begin{equation}
\label{eq:grad-norm-base}
\big\|\nabla_W\mathcal{L}\big\|_F^2 \;=\; \big\|X^\top G\big\|_F^2.
\end{equation}

Expanding the Frobenius norm into pairwise inner products (Appendix~\ref{app:grad-decomp}), and writing $z_i^\top$ and $g_i^\top$ for the $i$-th rows of $X$ and $G$, yields
\begin{equation}
\label{eq:grad-norm-pairwise-full}
\big\|\nabla_W\mathcal{L}\big\|_F^2
\;=\;
\sum_{i,j\in I} \big(z_i^\top z_j\big)\,\big(g_i^\top g_j\big).
\end{equation}

We now split token indices into conditioning and answer positions, $I=I_A\cup I_Q$, and separate Eq.~\eqref{eq:grad-norm-pairwise-full} into answer--answer, answer--conditioning, and conditioning--conditioning contributions:
\begin{equation}
\label{eq:block-decomp}
\begin{aligned}
\big\|\nabla_W\mathcal{L}\big\|_F^2
&=
\underbrace{\sum_{i,j\in I_A} (z_i^\top z_j)(g_i^\top g_j)}_{\text{AA}} \\
&\quad+
\underbrace{2\sum_{i\in I_A,\,j\in I_Q} (z_i^\top z_j)(g_i^\top g_j)}_{\text{AQ/QA}} \\
&\quad+
\underbrace{\sum_{i,j\in I_Q} (z_i^\top z_j)(g_i^\top g_j)}_{\text{QQ}}.
\end{aligned}
\end{equation}

This decomposition makes explicit how different parts of the input sequence jointly contribute to the update induced on the student.

\paragraph{Motivation.}
Eq.~\eqref{eq:block-decomp} shows that the update induced by a question--answer sequence depends on how answer tokens interact with one another and with the conditioning tokens.
Thus, for a fixed question, different verified teacher answers can induce different gradient interaction patterns and impose different update burdens on the current student.
We refer to this student-dependent update burden as \emph{student-centric learning cost}.
This cost is the quantity we seek to estimate when ranking candidate answers: lower-cost answers provide training signals that better match the current student.
Directly computing these blockwise gradient interactions requires backpropagated token gradients $g_i$ for every candidate answer, which is too expensive for large-scale answer selection.
We therefore derive an efficient forward-only proxy for student-centric learning cost in the next subsection.

\subsection{A Forward-only Proxy for Student-Centric Learning Cost}
\label{subsec:score}

Having defined student-centric learning cost as the update burden induced by a candidate answer, we now derive a scalable proxy for estimating it.
Guided by the block structure in Eq.~\eqref{eq:block-decomp}, we focus on the answer-involving AA and AQ/QA blocks.
For a fixed question, the QQ term exhibits much smaller cross-candidate variance than the answer-involving blocks, and therefore contributes little to answer ranking; empirical evidence is provided in Section~\ref{subsec:blockwise-std-analysis}.

The remaining challenge is to approximate these retained interactions using only forward-pass signals. 
We adopt a two-step reduction. First, we replace residual magnitudes in the retained last-layer interactions with token-level forward uncertainty, yielding a token-sensitive proxy that preserves the AA and AQ/QA structure. 
Second, we replace token-level uncertainty within each block by its mean, obtaining a more stable sequence-level proxy. 
Concretely, this leads to using mean token NLL to summarize candidate difficulty under the current student, while hidden-state interactions capture the retained AA and AQ/QA structure in a blockwise manner. This yields an efficient forward-only score for ranking multiple correct answers.
A theoretical analysis is provided in Appendix~\ref{app:theory}.

We instantiate the forward-only proxy using final-layer hidden states, for which the corresponding LM-head residuals admit a direct connection to token-level NLL; layer ablations in Section~\ref{subsec:layer-analysis} further support this choice.
Given $x=[q;\,a]$, we run the student once and extract the final-layer hidden states $H_\theta(x)=[h_t]_{t\in I}$.
We normalize each token representation as $u_t = h_t/\|h_t\|$, and define the answer and conditioning representation aggregates

{\small
\begin{equation}
\label{eq:mu-repeat}
\mu_A \;=\; \frac{1}{|I_A|}\sum_{i\in I_A} u_i,
\qquad
\mu_Q \;=\; \frac{1}{|I_Q|}\sum_{j\in I_Q} u_j.
\end{equation}}
We further define sequence-level difficulty signals using mean token NLL:
\begin{equation}
\label{eq:difficulty-repeat}
\begin{aligned}
d_A
&=
\frac{1}{|I_A|}
\sum_{i\in I_A}
-\log p_\theta(x_i \mid q, a_{<i}), \\
d_Q
&=
\frac{1}{|I_Q|}
\sum_{j\in I_Q}
-\log p_\theta(x_j \mid x_{<j}).
\end{aligned}
\end{equation}

Our score is based on a blockwise difficulty approximation: within each interaction block, we use a shared sequence-level difficulty scalar to modulate representation-side interactions. For the AA block, we define
\begin{equation}
\label{eq:score-aa}
C_{AA}(q,a;\theta)
\;=\;
d_A^2 \,\|\mu_A\|^2,
\end{equation}
where $\|\mu_A\|^2$ summarizes answer-side pairwise similarity structure. For the AQ block, we define
\begin{equation}
\label{eq:score-aq}
C_{AQ}(q,a;\theta)
\;=\;
d_A\,d_Q\,(\mu_A^\top \mu_Q),
\end{equation}
where $\mu_A^\top\mu_Q$ summarizes answer--conditioning interaction structure.

Putting these together, we define the student-centric learning cost as
{\small
\begin{equation}
\label{eq:score-repeat}
S(q,a;\theta)
=
(1-\lambda)C_{AA}(q,a;\theta)
+
\lambda C_{AQ}(q,a;\theta),
\end{equation}
}
where $\lambda \in [0,1]$ controls the trade-off between the two terms.
$S(q,a;\theta)$ is a forward-only proxy for ranking multiple already-correct answers under the current student. 
It quantifies how costly a correct answer is for the student to learn from at its current stage. 
Lower-cost candidates are interpreted as answers that the current student can absorb with less effort, and are therefore preferred for data selection.
A theoretical justification of the proposed score is provided in Appendix~\ref{app:theory}.

\begin{algorithm}[h]
\captionsetup{font=footnotesize, labelfont=bf}
\footnotesize
\caption{Student-Centric Answer Sampling (\Ours{})}
\label{alg:group-training}

\noindent\textbf{Require:}
\par\noindent
\begin{tabular}[t]{@{}l@{}}
training questions $\mathcal D$; candidate answer sets $\{\mathcal A(q)\}_{q\in\mathcal D}$; \\
number of groups $G$; initial student parameters $\theta_0$; \\
number of selection-training rounds $R$
\end{tabular}
\vspace{2pt}

\begin{algorithmic}[1]
\For{selection-training round $r=0,1,\ldots,R-1$}
    \ForAll{$q\in\mathcal D$}
        \State Compute $S(q,a;\theta_r)$ for all $a\in\mathcal A(q)$
        \State Sort $\mathcal A(q)$ in ascending order of $S(q,a;\theta_r)$
        \State Let $M_q=|\mathcal A(q)|$
        \State Write the sorted candidates as $a_{(1)},\ldots,a_{(M_q)}$
        \For{$g=0,1,\ldots,G$}
            \State Set $b_g=\left\lfloor gM_q/G\right\rfloor$
        \EndFor
        \For{$g=1,2,\ldots,G$}
            \State Set $\Pi_g(q)=\{a_{(k)}: b_{g-1}<k\le b_g\}$
        \EndFor
        \State Sample $\tilde a(q)\sim \mathrm{Unif}(\Pi_1(q))$
    \EndFor
    \State Form $\mathcal D_{\mathrm{low}}=\{(q,\tilde a(q)):q\in\mathcal D\}$
    \State Train the student on $\mathcal D_{\mathrm{low}}$ and obtain $\theta_{r+1}$
\EndFor
\end{algorithmic}
\end{algorithm}

\subsection{Student-Centric Answer Sampling}
\label{subsec:group-training}

Directly selecting the single lowest-score answer for each question can be brittle, as the exact per-question minimum may be sensitive to score noise, outliers, or small differences among near-tied candidates. We therefore adopt \Ours{}, a stratified sampling procedure that preserves low learning cost while improving robustness; ablation results supporting this design are provided in Section~\ref{subsec:group-number}.

At selection-training round $r$, let $\theta_r$ denote the current student parameters. For each question $q\in\mathcal D$, we compute $S(q,a;\theta_r)$ for all verified candidates $a\in\mathcal A(q)$ and sort the candidates in ascending order. Let $M_q=|\mathcal A(q)|$, and write the sorted candidates as $a_{(1)},a_{(2)},\dots,a_{(M_q)}$. We then partition the sorted list into $G$ consecutive groups $\Pi_1(q),\Pi_2(q),\dots,\Pi_G(q)$ using floor-based boundaries, where $\lfloor\cdot\rfloor$ denotes the floor function:
{\small
\begin{equation}
\Pi_g(q)
=
\Big\{
a_{(k)}:
\Big\lfloor \frac{(g-1)M_q}{G}\Big\rfloor < k \le
\Big\lfloor \frac{gM_q}{G}\Big\rfloor
\Big\}.
\end{equation}
}
Here, $g=1,\dots,G$.
By construction, $\Pi_1(q)$ is the lowest-score group.

We then sample one answer uniformly from the lowest-score group, $\tilde a(q)\sim \mathrm{Unif}(\Pi_1(q))$, for every question $q$, and construct the training set
\begin{equation}
\mathcal D_{\mathrm{low}}
=
\{(q,\tilde a(q)):\,q\in\mathcal D\}.
\end{equation}
The student is trained on $\mathcal D_{\mathrm{low}}$ to obtain the next iterate $\theta_{r+1}$. 
The full \Ours{} procedure is given in Algorithm~\ref{alg:group-training}.

\begin{table*}[h]
\caption{
\textbf{Overall Performance of Answer-selection Methods on Qwen2.5 Student Models.}
Higher is better for all metrics.
Best and second-best results within each student model are marked in \textcolor{red}{red} and \textcolor{blue}{blue}, respectively.
}
\vspace{-2mm}
\centering
\resizebox{0.95\textwidth}{!}{
\begin{tabular}{@{}llccccccc@{}}
\toprule
\textbf{Model} & \textbf{Method}
& \textbf{DeepScaleR}
& \textbf{GSM8K}
& \textbf{MATH}
& \textbf{OpenR1-Math}
& \textbf{IFEval}
& \textbf{LiveBench}
& \textbf{Avg.} \\
\midrule

\multirow{8}{*}{Qwen2.5-0.5B}
& RSR
& 7.13 $\pm$ 0.17
& 26.26 $\pm$ 0.35
& 21.62 $\pm$ 0.38
& 6.40 $\pm$ 0.24
& 22.77 $\pm$ 0.20
& 24.30 $\pm$ 0.21
& 18.08 \\

& GRACE
& 5.67 $\pm$ 0.17
& \textcolor{blue}{\textbf{37.65 $\pm$ 0.53}}
& \textcolor{blue}{\textbf{24.13 $\pm$ 0.31}}
& \textcolor{blue}{\textbf{8.30 $\pm$ 0.14}}
& 23.32 $\pm$ 0.19
& 26.90 $\pm$ 0.21
& \textcolor{blue}{\textbf{21.00}} \\

& RSR-Item
& 4.53 $\pm$ 0.42
& 16.10 $\pm$ 0.35
& 11.69 $\pm$ 0.32
& 7.77 $\pm$ 0.46
& 23.47 $\pm$ 0.28
& 26.80 $\pm$ 0.20
& 15.06 \\

& PPL
& 4.10 $\pm$ 0.14
& 6.07 $\pm$ 0.31
& 6.44 $\pm$ 0.22
& 6.23 $\pm$ 0.38
& 23.29 $\pm$ 0.23
& 26.30 $\pm$ 0.26
& 12.07 \\

& IFD
& 5.23 $\pm$ 0.48
& 7.00 $\pm$ 0.34
& 6.59 $\pm$ 0.21
& 6.03 $\pm$ 0.29
& \textcolor{blue}{\textbf{24.03 $\pm$ 0.19}}
& \textcolor{blue}{\textbf{27.10 $\pm$ 0.14}}
& 12.66 \\

& Router
& 6.63 $\pm$ 0.74
& 8.01 $\pm$ 0.81
& 9.75 $\pm$ 0.06
& 6.50 $\pm$ 0.50
& 23.84 $\pm$ 0.40
& 26.70 $\pm$ 0.52
& 13.57 \\

& Influence
& \textcolor{blue}{\textbf{8.77 $\pm$ 0.54}}
& 7.46 $\pm$ 0.16
& 22.98 $\pm$ 0.26
& 7.63 $\pm$ 0.80
& 23.29 $\pm$ 0.33
& 26.90 $\pm$ 0.34
& 16.17 \\

\rowcolor{gray!12}
& \Ours{}
& \textcolor{red}{\textbf{10.10 $\pm$ 0.85}}
& \textcolor{red}{\textbf{45.03 $\pm$ 0.22}}
& \textcolor{red}{\textbf{27.62 $\pm$ 0.07}}
& \textcolor{red}{\textbf{10.77 $\pm$ 0.19}}
& \textcolor{red}{\textbf{24.40 $\pm$ 0.16}}
& \textcolor{red}{\textbf{27.20 $\pm$ 0.17}}
& \textcolor{red}{\textbf{24.19}} \\

\midrule

\multirow{8}{*}{Qwen2.5-3B}
& RSR
& 21.53 $\pm$ 0.29
& 59.92 $\pm$ 0.35
& 45.02 $\pm$ 0.50
& 17.29 $\pm$ 0.45
& 56.56 $\pm$ 0.25
& \textcolor{blue}{\textbf{42.20 $\pm$ 0.35}}
& 40.42 \\

& GRACE
& \textcolor{blue}{\textbf{28.40 $\pm$ 0.65}}
& \textcolor{blue}{\textbf{76.80 $\pm$ 0.03}}
& \textcolor{blue}{\textbf{57.55 $\pm$ 0.50}}
& \textcolor{blue}{\textbf{30.90 $\pm$ 0.78}}
& 58.75 $\pm$ 0.48
& \textcolor{red}{\textbf{42.30 $\pm$ 0.55}}
& \textcolor{blue}{\textbf{49.12}} \\

& RSR-Item
& 20.07 $\pm$ 0.54
& 34.95 $\pm$ 0.59
& 26.91 $\pm$ 0.14
& 18.20 $\pm$ 0.37
& 59.15 $\pm$ 0.49
& 41.60 $\pm$ 0.30
& 33.48 \\

& PPL
& 19.40 $\pm$ 0.71
& 18.88 $\pm$ 0.22
& 16.57 $\pm$ 0.36
& 15.33 $\pm$ 1.03
& 60.40 $\pm$ 0.47
& 39.90 $\pm$ 0.50
& 28.41 \\

& IFD
& 26.80 $\pm$ 0.88
& 18.32 $\pm$ 0.04
& 16.12 $\pm$ 0.03
& 24.57 $\pm$ 0.82
& \textcolor{blue}{\textbf{60.48 $\pm$ 0.42}}
& 41.60 $\pm$ 0.27
& 31.32 \\

& Router
& 22.33 $\pm$ 1.11
& 23.20 $\pm$ 0.45
& 25.79 $\pm$ 0.53
& 18.70 $\pm$ 0.51
& 59.89 $\pm$ 0.55
& 41.70 $\pm$ 0.34
& 31.94 \\

& Influence
& 27.83 $\pm$ 0.73
& 25.22 $\pm$ 0.74
& 49.98 $\pm$ 0.43
& 28.05 $\pm$ 0.65
& 59.89 $\pm$ 0.64
& 41.60 $\pm$ 0.52
& 38.76 \\

\rowcolor{gray!12}
& \Ours{}
& \textcolor{red}{\textbf{33.20 $\pm$ 0.21}}
& \textcolor{red}{\textbf{82.82 $\pm$ 0.35}}
& \textcolor{red}{\textbf{61.69 $\pm$ 0.09}}
& \textcolor{red}{\textbf{33.20 $\pm$ 0.43}}
& \textcolor{red}{\textbf{61.37 $\pm$ 0.33}}
& \textcolor{red}{\textbf{42.30 $\pm$ 0.20}}
& \textcolor{red}{\textbf{52.43}} \\

\midrule

\multirow{8}{*}{Qwen2.5-7B}
& RSR
& 30.20 $\pm$ 0.57
& 62.55 $\pm$ 0.31
& 53.89 $\pm$ 0.31
& 27.20 $\pm$ 0.71
& 60.07 $\pm$ 0.41
& 47.40 $\pm$ 0.37
& 46.89 \\

& GRACE
& \textcolor{blue}{\textbf{38.97 $\pm$ 1.30}}
& \textcolor{blue}{\textbf{86.85 $\pm$ 0.47}}
& \textcolor{blue}{\textbf{65.71 $\pm$ 0.21}}
& \textcolor{blue}{\textbf{38.93 $\pm$ 0.24}}
& \textcolor{blue}{\textbf{62.70 $\pm$ 0.26}}
& 47.50 $\pm$ 0.18
& \textcolor{blue}{\textbf{56.78}} \\

& RSR-Item
& 24.37 $\pm$ 0.91
& 55.12 $\pm$ 0.65
& 37.40 $\pm$ 0.07
& 24.97 $\pm$ 0.54
& 61.00 $\pm$ 0.51
& \textcolor{blue}{\textbf{48.40 $\pm$ 0.46}}
& 41.88 \\

& PPL
& 18.03 $\pm$ 0.41
& 20.60 $\pm$ 0.47
& 18.04 $\pm$ 0.03
& 22.17 $\pm$ 0.82
& 62.29 $\pm$ 0.40
& 47.80 $\pm$ 0.43
& 31.49 \\

& IFD
& 17.73 $\pm$ 0.78
& 20.32 $\pm$ 0.24
& 18.25 $\pm$ 0.11
& 21.30 $\pm$ 0.62
& 61.92 $\pm$ 0.39
& 47.90 $\pm$ 0.26
& 31.24 \\

& Router
& 27.47 $\pm$ 1.13
& 21.99 $\pm$ 0.22
& 22.55 $\pm$ 0.43
& 27.33 $\pm$ 0.45
& 60.26 $\pm$ 0.33
& 47.10 $\pm$ 0.37
& 34.45 \\

& Influence
& 34.50 $\pm$ 0.98
& 19.61 $\pm$ 0.35
& 18.62 $\pm$ 0.31
& 34.50 $\pm$ 0.22
& 62.29 $\pm$ 0.29
& 47.60 $\pm$ 0.25
& 36.19 \\

\rowcolor{gray!12}
& \Ours{}
& \textcolor{red}{\textbf{41.97 $\pm$ 1.42}}
& \textcolor{red}{\textbf{89.38 $\pm$ 0.16}}
& \textcolor{red}{\textbf{69.91 $\pm$ 0.19}}
& \textcolor{red}{\textbf{42.13 $\pm$ 0.48}}
& \textcolor{red}{\textbf{68.58 $\pm$ 0.32}}
& \textcolor{red}{\textbf{49.60 $\pm$ 0.19}}
& \textcolor{red}{\textbf{60.26}} \\

\bottomrule
\end{tabular}
}
\vspace{-2mm}
\label{tab:overall-performance}
\end{table*}

\section{Experiments Setup}
\label{sec:experiments-setup}

 \paragraph{Models \& Datasets.}
For each math training question, we use 30 teacher models and keep only verified correct answers; for each open-ended training question, we use Llama-3.3-70B-Instruct~\citep{Dubey2024TheL3} to generate 10 diverse responses. See Appendix~\ref{app:candidate-verification} for more details.
We evaluate six instruction-tuned student models: Qwen2.5-0.5B-Instruct, Qwen2.5-3B-Instruct, Qwen2.5-7B-Instruct~\citep{Yang2024Qwen25TR}, Llama-3.2-1B-Instruct, Llama-3.2-3B-Instruct, and Meta-Llama-3-8B-Instruct~\citep{Dubey2024TheL3}. 
For mathematical dataset, we use MATH~\citep{Hendrycks2021MeasuringMP}, GSM8K~\citep{Cobbe2021TrainingVT}, DeepScaleR~\citep{deepscaler2025}, and OpenR1-Math~\citep{openr1update2025} as both training and evaluation sources. 
For open-ended instruction tuning, we train on Magpie~\citep{Xu2024MagpieAD} and evaluate on IFEval~\citep{Zhou2023InstructionFollowingEF} and LiveBench~\citep{White2024LiveBenchAC}. 
We report accuracy for mathematical tasks and official benchmark scores for open-ended tasks; when multiple runs are available, we report mean and standard deviation. 
Further implementation details are provided in Appendix~\ref{app:implet-detail}.

\paragraph{Baselines.}
We compare \Ours{} with perplexity (PPL)~\citep{Bengio2003ANP}, instruction-following difficulty (IFD)~\citep{Li2023FromQT}, RSR and its item-level variant (RSR-Item)~\citep{Yang2026WhichRT}, GRACE~\citep{Panigrahi2025InGG}, router-based teacher selection (Router)~\citep{Liu2026NoSB}, and influence-function-based data selection (Influence)~\citep{Humane2025InfluenceFF}. 
For likelihood-based scores whose selection direction is not fixed a priori, we evaluate both low-score and high-score variants. 
All methods are trained under the same student model, data budget, and fine-tuning configuration.
Further implementation details are provided in Appendix~\ref{app:implet-detail}.

\section{Results and Analysis}

\begin{figure*}[t]
\centering
\includegraphics[width=1.0\linewidth]{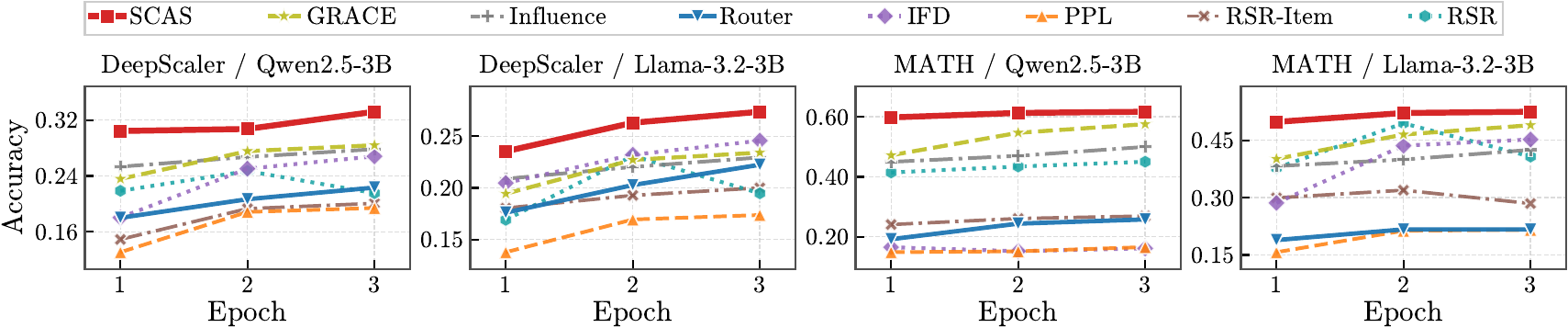}
\caption{\textbf{Data Efficiency Across Training Epochs.}}
\label{fig:data-efficiency}
\end{figure*}

\subsection{Overall Performance}
\label{subsec:overall-performance}

Table~\ref{tab:overall-performance} reports the main results on three Qwen2.5 student models. \Ours{} achieves the highest average score for all model sizes, improving over the strongest competing average by 3.19 points on Qwen2.5-0.5B, 3.31 points on Qwen2.5-3B, and 3.48 points on Qwen2.5-7B. The gains are especially clear on mathematical  benchmarks. Compared with the strongest baseline, \Ours{} improves GSM8K by 7.38 points for Qwen2.5-0.5B, 6.02 points for Qwen2.5-3B, and 2.53 points for Qwen2.5-7B. Similar improvements appear on MATH, where the gains are 3.49, 4.14, and 4.20 points, respectively, and on OpenR1-Math, where the gains are 2.47, 2.30, and 3.20 points.
The results also show that teacher- or difficulty-centric criteria are not uniformly reliable across student scales and tasks. GRACE is usually the strongest baseline, but other methods occasionally provide the second-best result, such as Influence on DeepScaleR for Qwen2.5-0.5B and RSR-Item on LiveBench for Qwen2.5-7B. In contrast, \Ours{} obtains the best or tied-best result in every model--metric column, with the only tie occurring on LiveBench for Qwen2.5-3B. \Ours{} also remains strong on open-ended instruction-following benchmarks, obtaining the best IFEval score for all three model sizes and the best or tied-best LiveBench score across all settings. These results support the central hypothesis of this work: among multiple correct teacher answers, selecting answers that are more learnable for the current student leads to better downstream training than relying only on teacher strength or generic data-difficulty scores.

\subsection{Efficiency Analysis}

\paragraph{Data Efficiency}

We next compare how different selection criteria affect learning progress over three training epochs.
Figure~\ref{fig:data-efficiency} reports performance after each training epoch on DeepScaleR and MATH for Qwen2.5-3B and Llama-3.2-3B students.
\Ours{} improves faster across datasets and student families, reaching higher accuracy after the first epoch and maintaining the lead through the final checkpoint.
This pattern suggests that selecting lower-learning-cost answers helps the student absorb supervision more efficiently at each stage of fine-tuning.
In contrast, difficulty-centric baselines improve more slowly or behave inconsistently across epochs, especially on MATH, where methods such as PPL and RSR-Item remain below the best-performing methods.
These results show that student-centric answer selection improves data efficiency by converting the same training data and optimization schedule into higher student accuracy throughout fine-tuning.

\paragraph{Computing Efficiency}

\begin{figure}[h]
\centering
\includegraphics[width=1.0\linewidth]{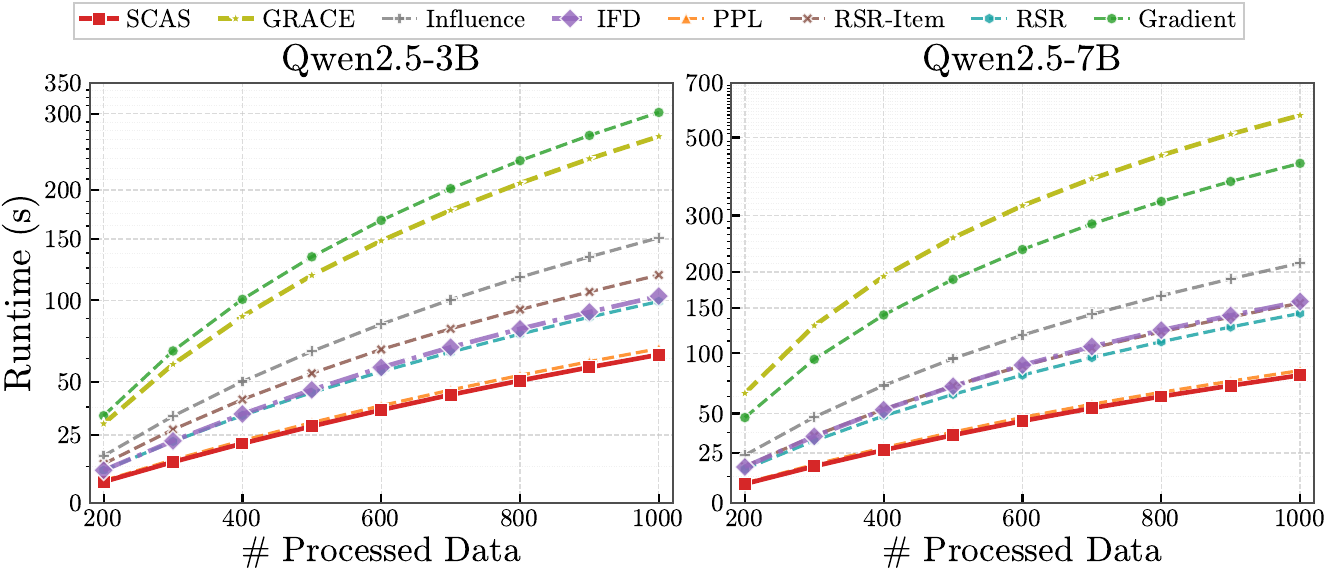}
\caption{
\textbf{Computing Efficiency for Selecting 1,000 Data Points.}
}
\label{fig:computing-efficiency}
\end{figure}

We further evaluate the computational efficiency of \Ours{}.
We measure this efficiency by testing the time required for different methods to select 1,000 training data points, where each question has 30 teacher-generated candidate answers 
We conduct this experiment on the MATH dataset with Qwen2.5-3B and Qwen2.5-7B.
The results are shown in Figure~\ref{fig:computing-efficiency}.
The experimental configuration follows Appendix~\ref{app:implet-detail}.
Figure~\ref{fig:computing-efficiency} compares \Ours{} with the baselines used in our experiments, including PPL, IFD, RSR, RSR-Item, Influence, and GRACE.
\Ours{} is computationally efficient across both model sizes.
At 1,000 processed data points, \Ours{} takes 64.9 seconds on Qwen2.5-3B and 79.9 seconds on Qwen2.5-7B, which is lower than the baselines.
This efficiency follows from our forward-only design: after a single forward pass, the score only requires token-level NLL values and lightweight hidden-state aggregation.
These results show that the proposed student-centric selection score is practical for big candidate pools, while retaining the training gains observed in the previous sections.

\subsection{Number of Groups}
\label{subsec:group-number}

\begin{figure}[h]
\centering
\includegraphics[width=1.0\linewidth]{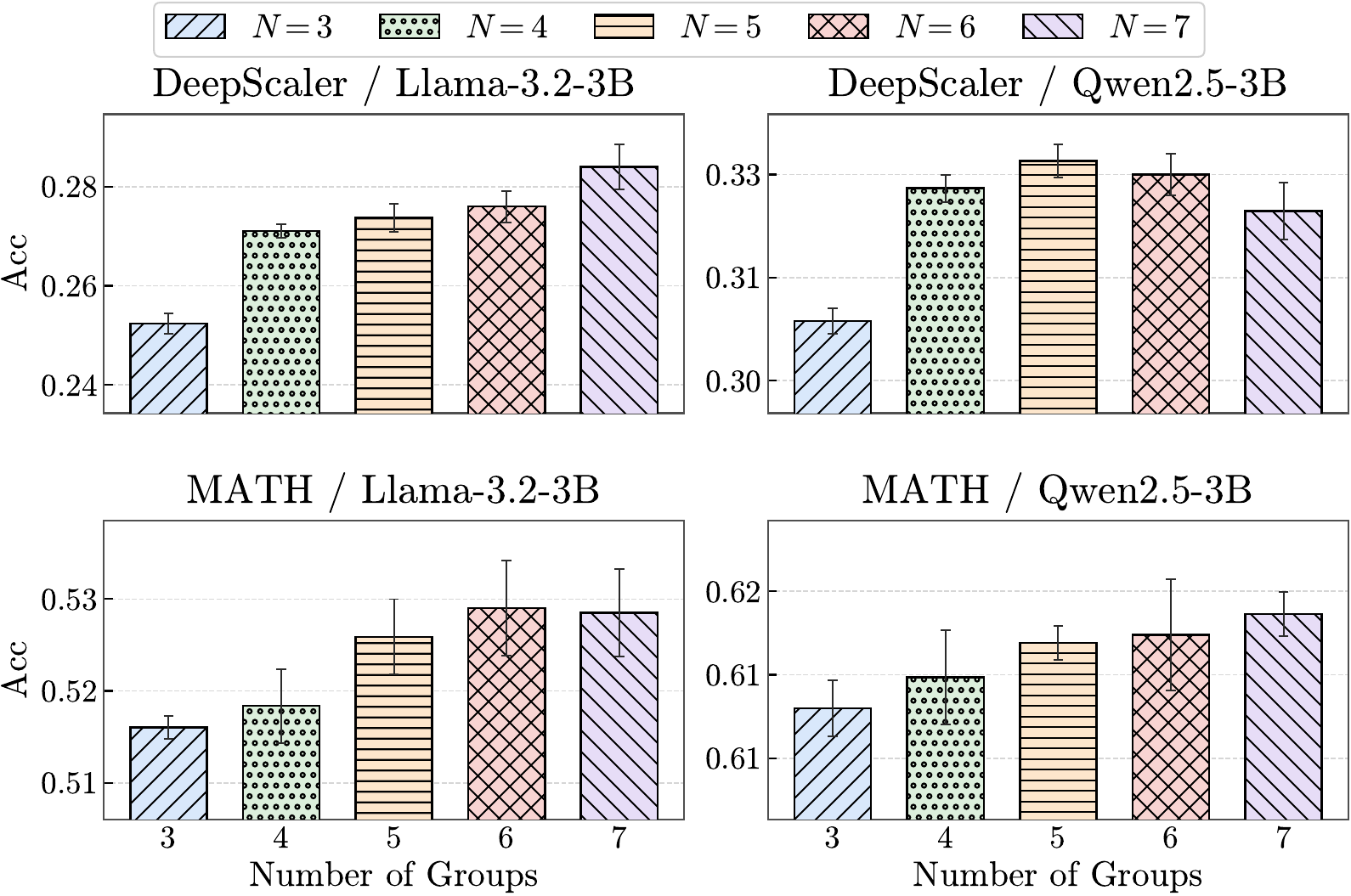}
\caption{
\textbf{Effect of the Number of Groups $G$ in \Ours{}.}
}
\label{fig:group-number}
\end{figure}

The number of groups $G$ controls the granularity of the stratified sampling step in Algorithm~\ref{alg:group-training}. 
Smaller $G$ keeps the lowest-score group $\Pi_1(q)$ more diverse, but may introduce noise by including candidates whose learning costs are only moderately low. 
Larger $G$ makes $\Pi_1(q)$ more selective, but can also make sampling more sensitive to score noise.
We study this effect on DeepScaleR and MATH using two  student models, Qwen2.5-3B and Llama-3.2-3B, by varying $G\in\{3,4,5,6,7\}$ under the same experimental setup as Section~\ref{subsec:overall-performance}. 
As shown in Figure~\ref{fig:group-number}, all settings improve when moving from coarse partitions ($G=3$ or $4$) to a moderately finer grouping. 
After that, the trend depends on the model and dataset: DeepScaleR with Llama-3.2-3B and MATH with Qwen2.5-3B continue to increase as $G$ grows, whereas DeepScaleR with Qwen2.5-3B and MATH with Llama-3.2-3B peak and then slightly decline. 
This suggests a trade-off between reducing coarse-group noise and maintaining a sufficiently robust sampling pool. 
We use $G=5$ as a stable default in the main experiments.

\subsection{Ablation on $\lambda$}
\label{subsec:lambda-ablation}

\begin{figure}[t]
\centering
\includegraphics[width=1.0\linewidth]{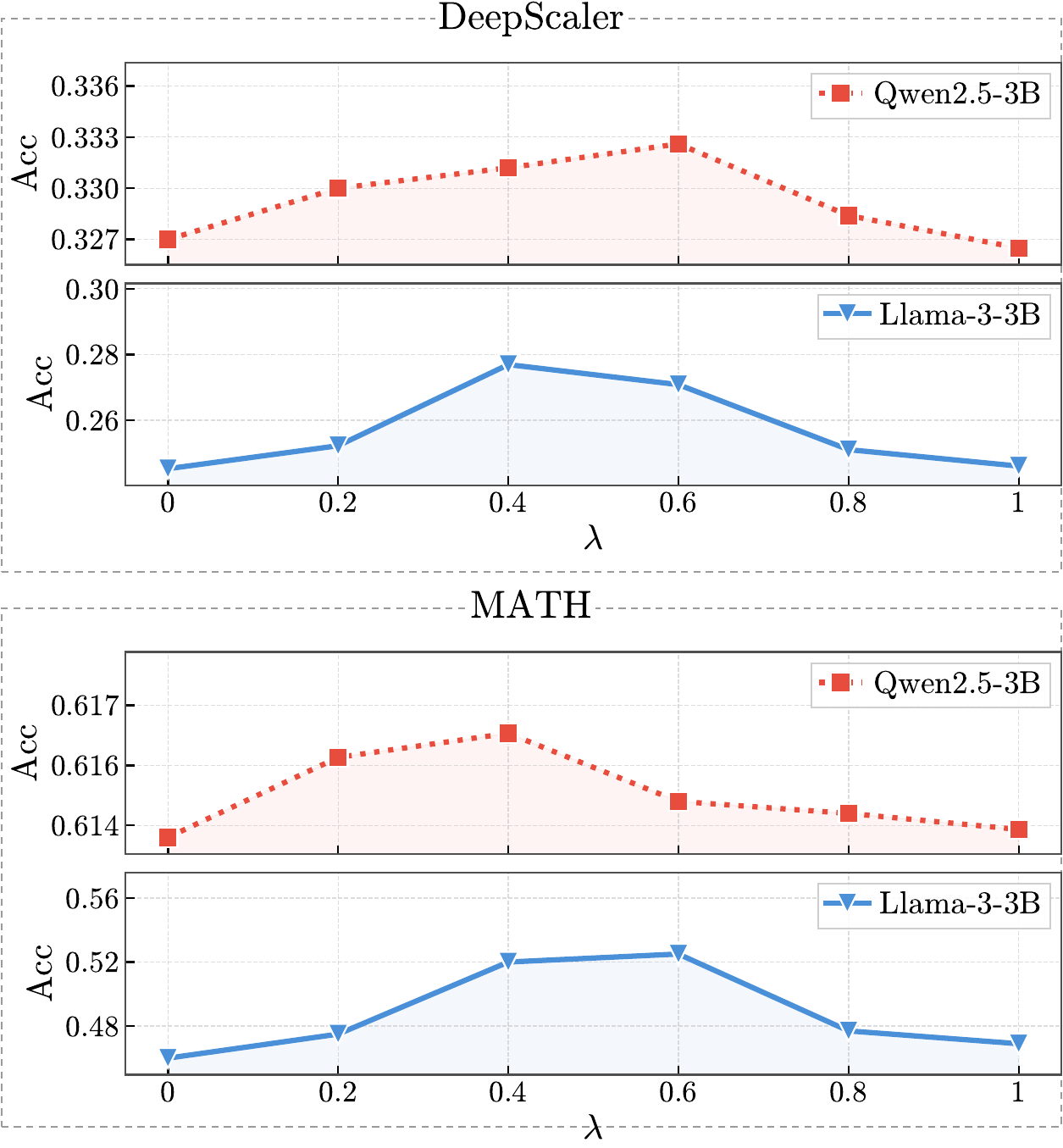}
\caption{
\textbf{Effect of the Trade-off Parameter $\lambda$}.
}
\label{fig:lambda-ablation}
\end{figure}

We ablate the trade-off parameter $\lambda$ in Eq.~\eqref{eq:score-repeat} to examine how the AA and AQ/QA terms contribute to student-centric answer selection. The two endpoints correspond to using only one component of the score: the left endpoint uses the AA term alone, while the right endpoint uses only the AQ/QA term. We vary $\lambda$ over $\{0,0.2,0.4,0.6,0.8,1.0\}$ and evaluate on DeepScaleR and MATH with Qwen2.5-3B and Llama-3.2-3B.
As shown in Figure~\ref{fig:lambda-ablation}, intermediate trade-offs generally lead to better performance than either endpoint. This trend is most pronounced for Llama-3.2-3B, whose accuracy peaks around 0.4 on DeepScaleR and around 0.6 on MATH. Qwen2.5-3B is less sensitive to this hyperparameter, but its best results also occur in the middle range, around 0.4--0.6. These patterns suggest that relying on either answer-internal interactions or answer--conditioning interactions alone is suboptimal.
The ablation supports our blockwise formulation of the learning-cost proxy: the AA and AQ/QA terms provide complementary signals, and combining them yields a more effective criterion for selecting learnable answers. We therefore use an intermediate trade-off in the main experiments.

\subsection{Ablation on the Learning Cost}
\label{subsec:learning-cost-ablation}

\begin{figure}[t]
\centering
\includegraphics[width=1.0\linewidth]{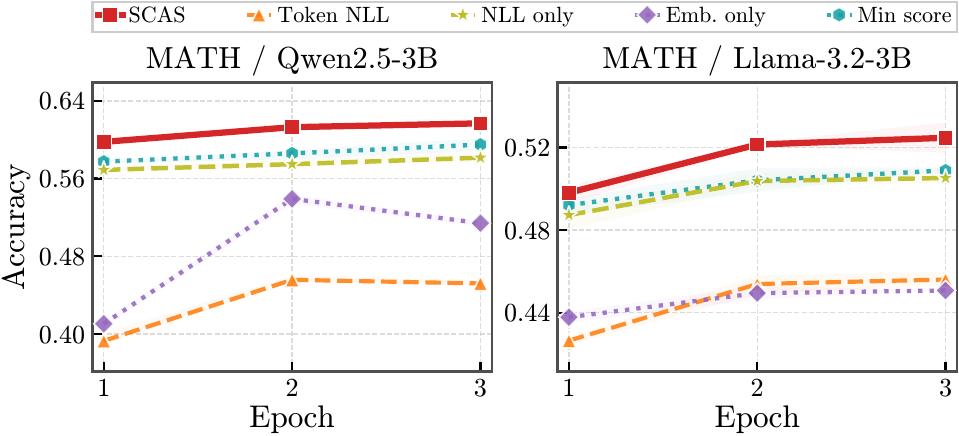}
\caption{
\textbf{Ablation of the Learning-cost Score.}
}
\label{fig:learning-cost-ablation}
\end{figure}

We ablate the learning-cost score on MATH with Qwen2.5-3B-Instruct and Llama-3.2-3B-Instruct, training each variant for three epochs. 
\Ours{} uses the product of sequence-level NLL and the embedding-based score.
We compare it with four controlled variants: Token NLL keeps the product form but replaces sequence-level NLL with token-level NLL; NLL only uses the sequence-level NLL term from the \Ours{} score; Emb. only uses the embedding-based term from the \Ours{} score; and Min score uses the full score but removes grouping, always selecting the single lowest-score answer for each question.
Detailed definitions of these variants are provided in Appendix~\ref{app:learning-cost-ablation-defs}.
The results are shown in Figure~\ref{fig:learning-cost-ablation}.
\Ours{} performs best for both student models throughout training.
The NLL-only and embedding-only variants both fall below the full score, showing that the two terms provide complementary signals for answer-level learnability.
Token NLL is consistently weaker, suggesting that local token difficulty does not capture the answer-level signal needed for selection.
Min score is competitive with NLL only but remains below \Ours{}, suggesting that always choosing the single lowest-score answer can be affected by score noise; this highlights the role of grouping in making low-cost answer selection more robust.

\subsection{Mean and Std of AA, AQ/QA, and QQ}
\label{subsec:blockwise-std-analysis}

\begin{table}[t]
\caption{\textbf{Mean and Standard Deviation of Blockwise Scores Across Candidate Answers.}}
\vspace{-2mm}
\centering
\scriptsize
\renewcommand{\arraystretch}{1.05}
\setlength{\tabcolsep}{2pt}
\begin{tabular*}{\linewidth}{@{\extracolsep{\fill}}ccl|ccc|ccc@{}}
\toprule
\textbf{Model} & \textbf{Size} & \textbf{Stat.} & \multicolumn{3}{c|}{\textbf{DeepScaleR}} & \multicolumn{3}{c}{\textbf{MATH}} \\
\cmidrule(lr){4-6}\cmidrule(l){7-9}
 & & & \textbf{AA} & \textbf{AQ/QA} & \textbf{QQ} & \textbf{AA} & \textbf{AQ/QA} & \textbf{QQ} \\
\midrule
\multirow{6}{*}{\rotatebox[origin=c]{90}{\footnotesize Qwen-2.5}}
& \multirow{2}{*}{0.5B} & Mean & 0.56 & 0.95 & 4.18 & 1.21 & 1.27 & 4.02 \\
&                       & Std. & 1.4465 & 0.8323 & 0.0055 & 2.9445 & 1.4772 & 0.0057 \\
& \multirow{2}{*}{3B}   & Mean & 0.91 & 1.04 & 4.62 & 2.60 & 1.59 & 4.36 \\
&                       & Std. & 2.7117 & 1.1836 & 0.0077 & 7.1118 & 2.3813 & 0.0077 \\
& \multirow{2}{*}{7B}   & Mean & 2.59 & 2.01 & 7.57 & 9.96 & 3.62 & 7.22 \\
&                       & Std. & 9.0950 & 2.8131 & 0.0067 & 29.1193 & 6.4291 & 0.0068 \\
\midrule
\multirow{6}{*}{\rotatebox[origin=c]{90}{\footnotesize Llama-3}}
& \multirow{2}{*}{1B}   & Mean & 1.32 & 2.93 & 10.31 & 3.11 & 5.73 & 10.47 \\
&                       & Std. & 3.7038 & 4.4042 & 0.0073 & 9.0740 & 11.3458 & 0.0077 \\
& \multirow{2}{*}{3B}   & Mean & 1.03 & 2.76 & 10.88 & 2.20 & 5.13 & 11.02 \\
&                       & Std. & 2.7242 & 3.7724 & 0.0049 & 6.1921 & 9.5294 & 0.0051 \\
& \multirow{2}{*}{8B}   & Mean & 1.06 & 3.41 & 14.07 & 2.30 & 6.08 & 14.28 \\
&                       & Std. & 2.5867 & 4.2268 & 0.0086 & 6.1960 & 10.5918 & 0.0088 \\
\bottomrule
\end{tabular*}
\vspace{-2mm}
\label{tab:blockwise-std}
\end{table}

The proxy in Eq.~\eqref{eq:score-repeat} keeps the answer-involving AA and AQ/QA blocks while omitting the conditioning-only QQ block. We validate this design choice by measuring how much each block varies across candidate answers for the same question. For each of DeepScaleR and MATH, we randomly sample 2,000 examples from the training set. On these examples, we compute the mean and standard deviation of AA, AQ/QA, and QQ using six student models: Qwen2.5-0.5B, Qwen2.5-3B, Qwen2.5-7B, Llama-3.2-1B, Llama-3.2-3B, and Llama-3-8B.
Table~\ref{tab:blockwise-std} shows that QQ is nearly constant across candidates: its standard deviation is at most about $0.01$ in all settings. In contrast, the answer-involving blocks exhibit much larger cross-candidate variation, with AA standard deviations ranging from $1.45$ to $29.12$ and AQ/QA standard deviations ranging from $0.83$ to $11.35$. Thus, QQ behaves mostly as a question-level offset, whereas AA and AQ/QA provide the main discriminative signals during data selection.

\section{Rank Correlation with Gradient Scores}
\label{subsec:rank-correlation-gradient}

We evaluate whether \Ours{} preserves the ranking induced by exact gradient values. 
For each question, we rank the candidate answers using each method and compute Spearman correlation with the ranking obtained from true gradient values; teacher-selection baselines follow the corresponding teacher-model ranking. 
We average over 1,000 questions from each of DeepScaleR, GSM8K, MATH, OpenR1-Math, and Magpie.
Table~\ref{tab:rank-correlation-gradient-by-size} reports mean correlations on five datasets, where DS, OR1, and Mag. denote DeepScaleR, OpenR1-Math, and Magpie, respectively. 
\Ours{} achieves the best average correlation for both Qwen2.5-3B and Qwen2.5-7B, reaching 0.783 and 0.844, and is best in every dataset column. 
These results indicate that \Ours{} better preserves the gradient-value preference order while remaining forward-only.

\begin{table}[t]
\caption{\textbf{Mean Spearman Rank Correlation with Gradient Values.}}
\vspace{-2mm}
\centering
\scriptsize
\renewcommand{\arraystretch}{1.03}
\setlength{\tabcolsep}{1pt}
\begin{tabular*}{\linewidth}{@{\extracolsep{\fill}}llcccccc@{}}
\toprule
\textbf{Model} & \textbf{Method} & \makebox[3.0em][c]{\textbf{DS}} & \makebox[3.0em][c]{\textbf{GSM8K}} & \makebox[3.0em][c]{\textbf{MATH}} & \makebox[3.0em][c]{\textbf{OR1}} & \makebox[3.0em][c]{\textbf{Mag.}} & \makebox[3.0em][c]{\textbf{Avg.}} \\
\midrule
\multirow{8}{*}{\rotatebox[origin=c]{90}{Qwen2.5-3B}} & GRACE     & \makebox[3.0em][r]{\phantom{-}0.781} & \makebox[3.0em][r]{\phantom{-}0.853} & \makebox[3.0em][r]{\phantom{-}0.829} & \makebox[3.0em][r]{\phantom{-}0.775} & \makebox[3.0em][r]{\phantom{-}0.348} & \makebox[3.0em][r]{\phantom{-}0.717} \\
 & Influence & \makebox[3.0em][r]{\phantom{-}0.604} & \makebox[3.0em][r]{\phantom{-}0.676} & \makebox[3.0em][r]{\phantom{-}0.662} & \makebox[3.0em][r]{\phantom{-}0.356} & \makebox[3.0em][r]{\phantom{-}0.261} & \makebox[3.0em][r]{\phantom{-}0.512} \\
 & IFD       & \makebox[3.0em][r]{\phantom{-}0.025} & \makebox[3.0em][r]{\phantom{-}0.053} & \makebox[3.0em][r]{-0.074} & \makebox[3.0em][r]{\phantom{-}0.101} & \makebox[3.0em][r]{\phantom{-}0.159} & \makebox[3.0em][r]{\phantom{-}0.053} \\
 & Router    & \makebox[3.0em][r]{\phantom{-}0.247} & \makebox[3.0em][r]{-0.582} & \makebox[3.0em][r]{-0.288} & \makebox[3.0em][r]{-0.393} & \makebox[3.0em][r]{-0.304} & \makebox[3.0em][r]{-0.264} \\
 & RSR       & \makebox[3.0em][r]{-0.795} & \makebox[3.0em][r]{-0.836} & \makebox[3.0em][r]{-0.811} & \makebox[3.0em][r]{-0.796} & \makebox[3.0em][r]{\phantom{-}0.221} & \makebox[3.0em][r]{-0.603} \\
 & RSR-Item  & \makebox[3.0em][r]{-0.802} & \makebox[3.0em][r]{-0.767} & \makebox[3.0em][r]{-0.791} & \makebox[3.0em][r]{-0.726} & \makebox[3.0em][r]{-0.074} & \makebox[3.0em][r]{-0.632} \\
 & PPL       & \makebox[3.0em][r]{-0.771} & \makebox[3.0em][r]{-0.842} & \makebox[3.0em][r]{-0.809} & \makebox[3.0em][r]{-0.791} & \makebox[3.0em][r]{-0.344} & \makebox[3.0em][r]{-0.712} \\
\cmidrule(lr){2-8}
 & \cellcolor{gray!12}Ours & \cellcolor{gray!12}\makebox[3.0em][r]{\textbf{\phantom{-}0.788}} & \cellcolor{gray!12}\makebox[3.0em][r]{\textbf{\phantom{-}0.883}} & \cellcolor{gray!12}\makebox[3.0em][r]{\textbf{\phantom{-}0.875}} & \cellcolor{gray!12}\makebox[3.0em][r]{\textbf{\phantom{-}0.795}} & \cellcolor{gray!12}\makebox[3.0em][r]{\textbf{\phantom{-}0.572}} & \cellcolor{gray!12}\makebox[3.0em][r]{\textbf{\phantom{-}0.783}} \\
\midrule
\multirow{8}{*}{\rotatebox[origin=c]{90}{Qwen2.5-7B}} & GRACE     & \makebox[3.0em][r]{\phantom{-}0.796} & \makebox[3.0em][r]{\phantom{-}0.891} & \makebox[3.0em][r]{\phantom{-}0.849} & \makebox[3.0em][r]{\phantom{-}0.804} & \makebox[3.0em][r]{\phantom{-}0.628} & \makebox[3.0em][r]{\phantom{-}0.793} \\
 & Influence & \makebox[3.0em][r]{\phantom{-}0.183} & \makebox[3.0em][r]{\phantom{-}0.820} & \makebox[3.0em][r]{\phantom{-}0.800} & \makebox[3.0em][r]{-0.089} & \makebox[3.0em][r]{\phantom{-}0.515} & \makebox[3.0em][r]{\phantom{-}0.446} \\
 & IFD       & \makebox[3.0em][r]{-0.360} & \makebox[3.0em][r]{-0.303} & \makebox[3.0em][r]{-0.483} & \makebox[3.0em][r]{-0.228} & \makebox[3.0em][r]{-0.374} & \makebox[3.0em][r]{-0.349} \\
 & Router    & \makebox[3.0em][r]{\phantom{-}0.285} & \makebox[3.0em][r]{-0.582} & \makebox[3.0em][r]{-0.266} & \makebox[3.0em][r]{-0.399} & \makebox[3.0em][r]{\phantom{-}0.241} & \makebox[3.0em][r]{-0.144} \\
 & RSR       & \makebox[3.0em][r]{-0.837} & \makebox[3.0em][r]{-0.636} & \makebox[3.0em][r]{-0.816} & \makebox[3.0em][r]{-0.817} & \makebox[3.0em][r]{\phantom{-}0.338} & \makebox[3.0em][r]{-0.553} \\
 & RSR-Item  & \makebox[3.0em][r]{-0.856} & \makebox[3.0em][r]{-0.722} & \makebox[3.0em][r]{-0.807} & \makebox[3.0em][r]{-0.787} & \makebox[3.0em][r]{\phantom{-}0.040} & \makebox[3.0em][r]{-0.626} \\
 & PPL       & \makebox[3.0em][r]{-0.856} & \makebox[3.0em][r]{-0.897} & \makebox[3.0em][r]{-0.871} & \makebox[3.0em][r]{-0.844} & \makebox[3.0em][r]{-0.724} & \makebox[3.0em][r]{-0.838} \\
\cmidrule(lr){2-8}
 & \cellcolor{gray!12}Ours & \cellcolor{gray!12}\makebox[3.0em][r]{\textbf{\phantom{-}0.861}} & \cellcolor{gray!12}\makebox[3.0em][r]{\textbf{\phantom{-}0.904}} & \cellcolor{gray!12}\makebox[3.0em][r]{\textbf{\phantom{-}0.883}} & \cellcolor{gray!12}\makebox[3.0em][r]{\textbf{\phantom{-}0.847}} & \cellcolor{gray!12}\makebox[3.0em][r]{\textbf{\phantom{-}0.726}} & \cellcolor{gray!12}\makebox[3.0em][r]{\textbf{\phantom{-}0.844}} \\
\bottomrule
\end{tabular*}
\vspace{-2mm}
\label{tab:rank-correlation-gradient-by-size}
\end{table}

\section{Layer Analysis}
\label{subsec:layer-analysis}
\begin{figure}[t]
\centering
\includegraphics[width=0.95\linewidth]{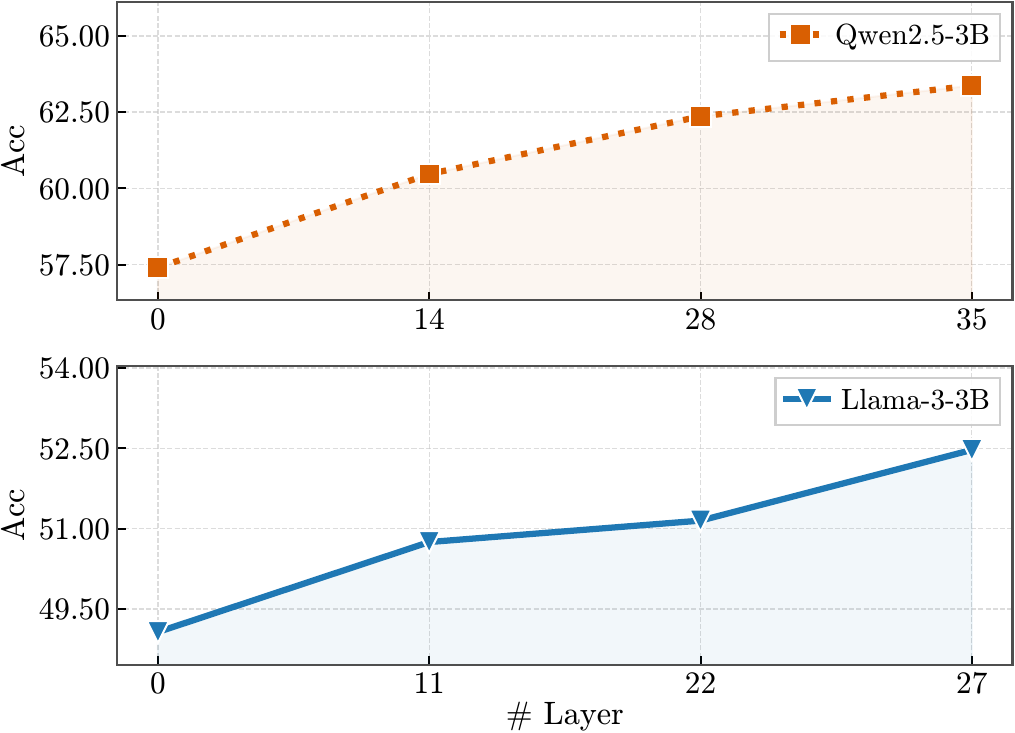}
\caption{
\textbf{Effect of the Representation Layer on Math.}
}
\label{fig:layer-analysis}
\end{figure}
In our main experiments, \Ours{} uses the final-layer hidden states of the student model to compute the answer-selection score. This follows our forward-only proxy: since token probabilities are computed from final-layer representations through the language modeling head, these representations are most directly connected to the NLL-based difficulty terms in our score.
To examine the effect of layer choice, we conduct a layer analysis on the Math setting with Llama-3.2-3B and Qwen2.5-3B. For each model, we select several approximately evenly spaced layers from shallow to deep, including the final layer, while keeping all other components unchanged. The results are shown in Figure~\ref{fig:layer-analysis}. Performance consistently improves with deeper layers for both models, with the final layer achieving the highest average accuracy. This supports our use of final-layer hidden states, suggesting that representations closest to the language modeling head better capture the student-centric signal needed for answer selection.


\section{Related Work}


\paragraph{LLM-Generated Synthetic Data and Distillation.}
LLM-generated supervision has become central to modern LLM training, from synthetic responses~\citep{Wang2022SelfInstructAL,Xu2026CoActCL, Xu2024MagpieAD} to reasoning traces, agent trajectories, and tool-use demonstrations~\citep{Li2025FromWT, Xu2025TOUCANS1,Wang2026RethinkingTE}. 
These signals are widely used for supervised fine-tuning, self-training, and distillation~\citep{Ko2025DistiLLM2AC, Hu2025UnveilingTL, Yang2026LearningBT}. 
Many pipelines therefore rely on the strongest available model as the teacher~\citep{Zhang2025FindYO, Wang2025LLMpoweredMF, Zhang2025TheBI}, implicitly treating teacher task performance as a proxy for data quality~\citep{Xu2025StrongerMA, Lei2026HumanLLMTP, He2026StudentintheLoopCD}. 
While recent work studies better synthetic data generation, knowledge transfer~\citep{Kessler2025TowardsAS, Wang2026SocialCoachPS, Shing2025TAIDTA}, and student-aware teacher selection~\citep{Panigrahi2025InGG}, we ask a finer-grained question: given multiple correct teacher-generated answers, which answer is most suitable for the current student?

\vspace{-2mm}

\paragraph{Data Selection for LLM Training.}
Our work also builds on data selection for model training. 
Classical methods select or reweight examples by quality, diversity, or estimated utility~\citep{Yoon2019DataVU, Sener2017ActiveLF}, while curriculum and self-paced learning emphasize the role of example difficulty in optimization~\citep{Bengio2009CurriculumL, Kumar2010SelfPacedLF,Hu2024EditableCB}. 
Gradient- and uncertainty-based criteria further estimate instance usefulness~\citep{Gal2017DeepBA, Koh2017UnderstandingBP}, and recent LLM studies adapt these ideas through synthetic data filtering, response ranking, adaptive sampling~\citep{Hu2023HowMV, Xia2024LESSSI, Hu2023LeveragingRG, Hu2024LetsAG,Hu2024SemiSupervisedCB}, and student-aware teacher or trajectory selection~\citep{Yang2026WhichRT}. 
Most of this work selects questions, examples, trajectories, or datasets~\citep{Zhang2023OnTT, Zhang2025EDGEED, Liu2025GuardReasonerTR}; our setting instead selects among multiple verified correct answers to the same question. 
We model this answer-level choice through student-centric learning cost and instantiate it with a forward-only proxy in \Ours{}.

\section{Conclusion}
This work shows that teacher competence is not always a reliable proxy for teaching quality in teacher-generated LLM supervision. Even among correct answers, the strongest teacher may not provide the best supervision for a given student. We propose \Ours{}, a forward-only proxy for student-centric learning cost that selects verified teacher-generated answers matched to the current student. Across 30 teacher models, 6 student base models, and 6 tasks, \Ours{} consistently outperforms teacher-centric and difficulty-based selection, highlighting the importance of student-matched supervision over teacher strength alone.

\section{Limitations}
This work has three main limitations. First, our experiments primarily focus on 3B- and 7B-scale instruction-tuned students, and the relationship between teacher strength and student-centric supervision may differ for substantially larger students, stronger base models, or architectures with different training dynamics. Second, although our teacher pool includes 30 models, it is not exhaustive, and a larger or more diverse pool may affect our conclusions.

\section{Acknowledgement}

This work is partially supported by the Air Force Office
of Scientific Research (AFOSR) under grant FA9550-23-1-
0208, and the National Science Foundation (NSF) under grants
IIS 2229876.
This work is supported in part by funds provided by the
National Science Foundation, by the Department of Homeland
Security, and by IBM. Any opinions, findings, and conclusions
or recommendations expressed in this material are
those of the author(s) and do not necessarily reflect the views
of the National Science Foundation or its federal agency and
industry partners.

\clearpage

\bibliography{custom}

\clearpage

\appendix

\section{Implementation Details}
\label{app:implet-detail}

\subsection{Experimental Setup}
\label{app:experimental-setup}

All experiments were conducted on a high-performance computing server equipped
with an AMD EPYC 7V12 64-Core Processor, 512 GB of system memory, 96 logical
CPUs, and eight NVIDIA A100 GPUs with 40 GB of memory each. Unless otherwise
specified, abbreviated student-model names refer to the following Hugging Face
checkpoints: Qwen2.5-0.5B-Instruct\footnote{\url{https://huggingface.co/Qwen/Qwen2.5-0.5B-Instruct}},
Qwen2.5-3B-Instruct\footnote{\url{https://huggingface.co/Qwen/Qwen2.5-3B-Instruct}},
Qwen2.5-7B-Instruct\footnote{\url{https://huggingface.co/Qwen/Qwen2.5-7B-Instruct}},
Llama-3.2-1B-Instruct\footnote{\url{https://huggingface.co/meta-llama/Llama-3.2-1B-Instruct}},
Llama-3.2-3B-Instruct\footnote{\url{https://huggingface.co/meta-llama/Llama-3.2-3B-Instruct}},
and Meta-Llama-3-8B-Instruct\footnote{\url{https://huggingface.co/meta-llama/Meta-Llama-3-8B-Instruct}}.
All student models are trained using full-parameter supervised fine-tuning
under the same optimization schedule. For \Ours{}, we use three
selection--training rounds and recompute the candidate scores using the
updated student between consecutive rounds. Because \Ours{} samples uniformly
from the lowest-score candidate group, we repeat the experiments using random
seeds 0, 1, and 2 and report the mean and standard deviation across runs.
Table~\ref{tab:training-hyperparameters} summarizes the training and teacher
generation settings.
Table~\ref{tab:data-reproducibility} reports the training-set sizes, evaluation
sets and versions, and candidate-pool statistics. For mathematical datasets,
the candidate count is the mean number of answers retained per question after
correctness verification. For Magpie, it denotes the fixed number of responses
generated per instruction. 
The common-correct subsets in Table~\ref{tab:data-reproducibility} are used
only for the controlled experiment in Figure~\ref{fig:intro_image}, where we
retain the intersection of questions answered correctly by all teachers. In
the main experiments with the 30-teacher pool, we retain the full set of
training questions and discard only incorrect candidate answers within each
question.
For the baselines, we follow prior settings and released implementations.
Specifically, for perplexity (PPL)~\citep{Bengio2003ANP} and
instruction-following difficulty (IFD)~\citep{Li2023FromQT}, we follow the
evaluation setup of \citet{Xu2025StrongerMA}; for RSR and RSR-Item, we follow
\citet{Yang2026WhichRT}; for GRACE, we follow the released implementation of
\citet{Panigrahi2025InGG}; for router-based teacher selection, we follow
\citet{Liu2026NoSB}; and for influence-function-based selection, we follow
\citet{Humane2025InfluenceFF}.
For \Ours{} scoring, we use the final-layer hidden states of the current student
to compute the representation aggregates \(\mu_A\) and \(\mu_Q\). The final
layer achieves the best answer-selection performance in our layer ablation and
is closest to the LM-head residuals motivating our forward proxy. The complete
layer-ablation results are reported in Section~\ref{subsec:layer-analysis}.
Our code is publicly available\footnote{\url{https://github.com/ppsmk388/Student-Centric-Answer-Selection}}.

\begin{table*}[t]
\centering
\caption{Training and teacher-generation settings used in the main
experiments.}
\small
\begin{tabular}{@{}p{0.27\textwidth}p{0.68\textwidth}@{}}
\toprule
\textbf{Setting} & \textbf{Value} \\
\midrule
Fine-tuning type
& Full-parameter supervised fine-tuning \\
Training framework
& LLaMA-Factory with Transformers and DeepSpeed ZeRO-3 \\
Optimizer
& AdamW using the LLaMA-Factory/Transformers defaults \\
Learning rate
& \(1.0\times10^{-5}\) \\
Learning-rate scheduler
& Cosine decay with a warmup ratio of \(0.05\) \\
Training epochs
& 3 \\
\Ours{} selection--training rounds
& 3 \\
Maximum sequence length
& 4,096 tokens \\
Random seeds
& \(0, 1, 2\) \\
Effective batch size
& 64 for Qwen2.5-0.5B and Llama-3.2-1B; 48 for Qwen2.5-3B and
Llama-3.2-3B; 16 for Qwen2.5-7B and Llama-3-8B \\
Number of candidate groups
& \(G=5\) \\
Math teacher decoding
& \texttt{temperature}$=0.7$ and \texttt{max\_tokens}$=4096$ \\
Magpie teacher decoding
& \texttt{temperature}$=0.7$ and \texttt{top\_p}$=0.9$; 10 responses
generated per instruction \\
\bottomrule
\end{tabular}
\label{tab:training-hyperparameters}
\end{table*}

\begin{table*}[t]
\centering
\small
\caption{Dataset sizes, evaluation sets, and candidate-pool statistics.
The last column reports the common-correct subsets used only in the controlled
teacher--student comparison in Figure~\ref{fig:intro_image}.}
\begin{tabular}{@{}l r p{0.32\textwidth} c c@{}}
\toprule
\textbf{Dataset}
& \textbf{Train}
& \textbf{Evaluation set/version}
& \textbf{Candidates/item}
& \textbf{Fig.~\ref{fig:intro_image} subset} \\
\midrule
GSM8K
& 7,473
& Official test split (1,319)
& 26.60
& 6,739 / 7,473 \\

MATH
& 7,500
& Official test split (5,000)
& 25.41
& 4,951 / 7,500 \\

DeepScaleR
& 39,315
& Held-out evaluation set (1,000)
& 16.85
& 11,258 / 39,315 \\

OpenR1-Math
& 50,000
& Held-out evaluation set (1,000)
& 18.17
& -- \\

Magpie
& 80,000
& IFEval official set (541); LiveBench commit
\texttt{3839eb0} (1,000)
& 10 generated
& -- \\
\bottomrule
\end{tabular}
\label{tab:data-reproducibility}
\end{table*}

\subsection{Teacher Model Pool for Mathematical Tasks}
\label{app:teacher-model-pool}

For the mathematical tasks, namely MATH~\citep{Hendrycks2021MeasuringMP}, GSM8K~\citep{Cobbe2021TrainingVT}, DeepScaleR~\citep{deepscaler2025}, and OpenR1-Math~\citep{openr1update2025}, we collect candidate answers from a pool of 30 teacher models. 
The same teacher pool is used to generate candidate answers for all mathematical datasets, and all answer-selection methods operate on the same verified candidate pool after correctness filtering. 
The complete list of teacher models is shown in Table~\ref{tab:teacher-model-pool}.

\begin{table*}[t]
\centering
\small
\caption{\textbf{Teacher Model Pool for Mathematical Candidate Generation.}
We use 30 open- and closed-source teacher models to generate candidate answers for mathematical tasks.}
\vspace{-1mm}
\begin{tabular}{lll}
\toprule
\textbf{Model} & \textbf{Model} & \textbf{Model} \\
\midrule
\texttt{gemma-3-12b-it} & \texttt{gemma-3-27b-it} & \texttt{gemma-4-31b-it} \\
\texttt{gpt-4.1\_2025-04-14} & \texttt{gpt-4o\_2024-08-06} & \texttt{gpt-4o\_2024-11-20} \\
\texttt{gpt-5-chat\_2025-08-07} & \texttt{gpt-5-chat\_2025-10-03} & \texttt{gpt-5-nano\_2025-08-07} \\
\texttt{gpt-5.1-chat\_2025-11-13} & \texttt{gpt-5.2-chat\_2025-12-11} & \texttt{gpt-5\_2025-08-07} \\
\texttt{gpt-oss-120b} & \texttt{gpt-oss-20b} & \texttt{grok-3\_1} \\
\texttt{grok-4-fast-non-reasoning\_1} & \texttt{grok-4\_1} & \texttt{llama-3.3-70b-instruct} \\
\texttt{Llama-3.1-70B-Instruct} & \texttt{mistral-3-14b-instruct} & \texttt{nvidia-llama3.3-nemotron-49b-v1.5} \\
\texttt{olmo-3.1-32b-instruct} & \texttt{phi-4-reasoning-plus} & \texttt{qwen2.5-14b-instruct} \\
\texttt{qwen2.5-32b-instruct} & \texttt{qwen2.5-72b-instruct} & \texttt{qwen2.5-coder-32b} \\
\texttt{qwen2.5-math-72b} & \texttt{qwen3-32b} & \texttt{qwen3.5-27b} \\
\bottomrule
\end{tabular}
\vspace{-2mm}
\label{tab:teacher-model-pool}
\end{table*}

\subsection{Candidate Answer Verification}
\label{app:candidate-verification}

In this work, a candidate answer is considered \emph{verified} if it passes an automatic correctness check appropriate to the target task. 
The purpose of verification is to ensure that answer selection is performed among valid supervision signals rather than confounding the selection criterion with answer correctness. 
That is, our method does not attempt to choose between correct and incorrect answers; instead, it ranks and samples among candidate answers that have already been verified as acceptable for training.

For mathematical reasoning tasks, including MATH~\citep{Hendrycks2021MeasuringMP}, GSM8K~\citep{Cobbe2021TrainingVT}, DeepScaleR~\citep{deepscaler2025}, and OpenR1-Math~\citep{openr1update2025}, candidate answers are generated by the 30-model teacher pool in Appendix~\ref{app:teacher-model-pool}. We verify them by comparing the predicted final answer with the reference answer and filter out candidates that are judged incorrect. 
Only the remaining correct candidates are retained in the candidate pool \(\mathcal A(q)\) for each question \(q\). 
This filtering step ensures that all answer-selection methods operate on the same set of verified correct mathematical solutions.

For the open-ended instruction-following task Magpie~\citep{Xu2024MagpieAD}, exact correctness verification is not well defined because there is no single reference answer. 
Therefore, instead of applying answer-level correctness filtering, we generate 10 candidate responses for each instruction using the same teacher model, Llama-3.3-70B-Instruct. 
For generating diverse responses from LLMs, we follow the configuration in \citet{Yuan2024SelfRewardingLM}, setting the sampling temperature to \(T=0.7\) and nucleus sampling probability to \(p=0.9\). 
Since all Magpie candidates are generated by the same model under the same decoding configuration, we assume that their overall response quality is comparable and that they have a similar degree of reliability for open-ended supervision. 
Under this setting, candidate selection mainly reflects student-centric learnability rather than differences in teacher identity or answer correctness.

\subsection{Introductory Teacher--Student Comparison}
\label{app:intro-figure-detail}

Figure~\ref{fig:intro_image} visualizes the relation between teacher task accuracy and downstream student accuracy in a controlled setting.
We use nine selected teacher models: {gpt-oss-120b}, {olmo-3.1-32b-instruct}, {qwen3.5-27b}, {qwen2.5-72b}, {llama-3.3-70b}, {gemma-4-31b-it}, {phi-4-reasoning}, {gpt-5-chat\_2025-10-03}, and {qwen3-32b}.
We evaluate the teachers and students on three mathematical datasets: DeepScaleR, GSM8K, and MATH.
The student models in this visualization are Qwen2.5-0.5B-Instruct, Qwen2.5-3B-Instruct, and Qwen2.5-7B-Instruct.

For the x-axis, each teacher is evaluated directly on the corresponding test set, and the resulting accuracy is used as its teacher accuracy.
For the y-axis, each teacher first generates answers for the student training set.
To control for correctness and make the teacher comparison depend on the form of supervision rather than on different error rates in the training data, we construct the largest common correct subset of training questions: an instance is retained only if all nine selected teachers generate a correct answer for it.
Each teacher then contributes its own generated answer on this same retained subset, and a Qwen2.5 student is fine-tuned on the resulting teacher-specific synthetic training set.
The plotted student accuracy is the student's test accuracy after this fine-tuning step.
This protocol ensures that all teachers are compared on the same questions and that every plotted training answer is correct.

For each dataset--student subplot, the displayed correlation coefficient is
Pearson's~\citep{PearsonVIINO} \(r\) computed over the nine teachers. Let \(x_i\) denote the test
accuracy of teacher \(i\), and let \(y_i\) denote the test accuracy of the
student fine-tuned on teacher \(i\)'s synthetic answers. We compute
\begin{equation}
r =
\frac{
\sum_{i=1}^{9} (x_i-\bar{x})(y_i-\bar{y})
}{
\sqrt{\sum_{i=1}^{9} (x_i-\bar{x})^2}
\sqrt{\sum_{i=1}^{9} (y_i-\bar{y})^2}
},
\end{equation}
where \(\bar{x}\) and \(\bar{y}\) are the means over the nine teachers.

\subsection{Learning-Cost Ablation Variants}
\label{app:learning-cost-ablation-defs}

We define the variants used in the learning-cost ablation in Section~\ref{subsec:learning-cost-ablation}. 
All variants use the same candidate-answer pool, correctness filtering, student models, and training schedule; they differ only in how candidate answers are scored or sampled.
Let \(E_{AA}(q,a;\theta)=\|\mu_A\|^2\) and \(E_{AQ}(q,a;\theta)=\mu_A^\top\mu_Q\). 
The full \Ours{} score used before grouping is
{\small
\begin{equation}
\begin{aligned}
S_{\mathrm{full}}(q,a;\theta)
&=
(1-\lambda)d_A^2 E_{AA}(q,a;\theta) \\
&\quad+
\lambda d_A d_Q E_{AQ}(q,a;\theta).
\end{aligned}
\end{equation}
}

For the token-level variant, let
\(\ell_i=-\log p_\theta(x_i\mid q,a_{<i})\) for \(i\in I_A\) and
\(\ell_j=-\log p_\theta(x_j\mid x_{<j})\) for \(j\in I_Q\). 
Instead of replacing token-level NLLs by their sequence means \(d_A\) and \(d_Q\),

\noindent \textbf{Token NLL} keeps token-level weights inside the interaction terms:
{\small
\begin{equation}
\bar v_A=\frac{1}{|I_A|}\sum_{i\in I_A}\ell_i u_i,
\qquad
\bar v_Q=\frac{1}{|I_Q|}\sum_{j\in I_Q}\ell_j u_j,
\end{equation}
}
{\small
\begin{equation}
S_{\mathrm{tok}}(q,a;\theta)
=
(1-\lambda)\|\bar v_A\|^2
+
\lambda\,\bar v_A^\top \bar v_Q .
\end{equation}
}

\noindent \textbf{NLL only} removes the embedding-side factors and uses
{\small
\begin{equation}
S_{\mathrm{NLL}}(q,a;\theta)
=
(1-\lambda)d_A^2+\lambda d_A d_Q .
\end{equation}
}

\noindent \textbf{Emb. only} removes the NLL factors and uses
{\small
\begin{equation}
\begin{aligned}
S_{\mathrm{emb}}(q,a;\theta)
&=
(1-\lambda)E_{AA}(q,a;\theta) \\
&\quad+
\lambda E_{AQ}(q,a;\theta).
\end{aligned}
\end{equation}
}
For Token NLL, NLL only, and Emb. only, we sort candidates by the corresponding score and apply the same grouping-and-sampling procedure as \Ours{}.

\noindent \textbf{Min score} keeps \(S_{\mathrm{full}}\) but removes grouping:
{\small
\begin{equation}
\tilde a_{\min}(q)
=
\operatorname*{arg\,min}_{a\in\mathcal A(q)}
S_{\mathrm{full}}(q,a;\theta_r).
\end{equation}
}
Thus Min score always selects the single lowest-score candidate answer for each question.

\subsection{Notation}
\label{subsec:notation}

Table~\ref{tab:notA} and Table~\ref{tab:notB} summarize the main notation used in the paper. Auxiliary symbols used only in appendix proofs are defined locally when introduced.

\begin{table*}[t]
\caption{\textbf{Summary of Notation (A): Data, Sequences, Student model, and Gradient Decomposition.}}
\vspace{-2mm}
\centering
\scriptsize
\renewcommand{\arraystretch}{1.05}
\setlength{\tabcolsep}{4pt}
\begin{tabular}{@{}p{0.24\textwidth}p{0.70\textwidth}@{}}
\toprule
\textbf{Symbol} & \textbf{Description} \\
\midrule
\multicolumn{2}{@{}l}{\textbf{Part I: Data and training sequences}} \\
\midrule
$\mathcal{D}=\{q_i\}_{i=1}^N$ & Training set of questions, where $N$ is the number of questions. \\
$q$ & A question or prompt from $\mathcal{D}$. \\
$\mathcal{A}(q)$ & Set of verified correct teacher-generated candidate answers for question $q$. \\
$a$ & A candidate answer in $\mathcal{A}(q)$. \\
$x=[q;\,a]$ & Training-templated sequence formed by concatenating the question and answer. \\
$x_t$ & Token at position $t$ in $x$. \\
$x_{<t}$ & Prefix tokens before position $t$. \\
$a_{<i}$ & Answer-side prefix before answer token position $i$. \\
$I=\{1,\dots,|x|\}$ & Set of token positions in $x$. \\
$I_Q$ & Token positions corresponding to the question or conditioning part. \\
$I_A$ & Token positions corresponding to the answer part. \\
$|q|, |x|, |I_Q|, |I_A|$ & Sequence or block lengths. \\
\midrule
\multicolumn{2}{@{}l}{\textbf{Part II: Student model and gradient decomposition}} \\
\midrule
$p_\theta(\cdot)$ & Autoregressive student language model. \\
$\theta$ & Student model parameters. \\
$\theta_r$ & Student parameters at selection-training round $r$. \\
$\mathbb{R}$ & Set of real numbers. \\
$h_t$ & Final-layer hidden state at token position $t$. \\
$H_\theta(x)=[h_t]_{t\in I}$ & Final-layer hidden-state sequence produced by the student on $x$. \\
$W\in\mathbb{R}^{d\times m}$ & Generic student weight matrix used in the gradient-decomposition analysis. \\
$d,m$ & Input and output dimensions of the generic weight matrix $W$. \\
$X=[z_t^\top]_{t\in I}$ & Token representations feeding into $W$, where $z_t$ is the representation at position $t$. \\
$A=XW$ & Pre-activation matrix corresponding to $W$; this local matrix $A$ is distinct from the answer set $\mathcal{A}(q)$. \\
$\mathcal{L}$ & Training loss for a sequence. \\
$G=[g_t^\top]_{t\in I}$ & Gradient of $\mathcal{L}$ with respect to $A$, where $g_t$ is the token-level gradient at position $t$. \\
$\nabla_W\mathcal{L}$ & Gradient of the loss with respect to $W$. \\
$\|\cdot\|_F$ & Frobenius norm. \\
$(\cdot)^\top$ & Vector or matrix transpose. \\
AA & Answer--answer interaction block, involving pairs $i,j\in I_A$. \\
AQ/QA & Answer--conditioning interaction block, involving one answer position and one question position. \\
QQ & Conditioning--conditioning interaction block, involving pairs $i,j\in I_Q$. \\
\bottomrule
\end{tabular}
\vspace{-2mm}
\label{tab:notA}
\end{table*}

\begin{table*}[t]
\caption{\textbf{Summary of Notation (B): Forward-only Learning-cost Proxy and Student-centric Answer Sampling.}}
\vspace{-2mm}
\centering
\scriptsize
\renewcommand{\arraystretch}{1.05}
\setlength{\tabcolsep}{4pt}
\begin{tabular}{@{}p{0.24\textwidth}p{0.70\textwidth}@{}}
\toprule
\textbf{Symbol} & \textbf{Description} \\
\midrule
\multicolumn{2}{@{}l}{\textbf{Part III: Forward-only learning-cost proxy}} \\
\midrule
$u_t=h_t/\|h_t\|$ & Normalized final-layer hidden state at token position $t$. \\
$\mu_A$ & Mean normalized hidden representation over answer positions, $\mu_A=\frac{1}{|I_A|}\sum_{i\in I_A}u_i$. \\
$\mu_Q$ & Mean normalized hidden representation over question or conditioning positions, $\mu_Q=\frac{1}{|I_Q|}\sum_{j\in I_Q}u_j$. \\
$d_A$ & Mean token negative log-likelihood (NLL) over answer positions. \\
$d_Q$ & Mean token NLL over question or conditioning positions. \\
$p_\theta(x_i\mid q,a_{<i})$ & Student probability assigned to answer token $x_i$ given the question and previous answer tokens. \\
$p_\theta(x_j\mid x_{<j})$ & Student probability assigned to conditioning token $x_j$ given its prefix. \\
$C_{AA}(q,a;\theta)$ & Answer--answer component of the proxy, $d_A^2\|\mu_A\|^2$. \\
$C_{AQ}(q,a;\theta)$ & Answer--conditioning component of the proxy, $d_A d_Q(\mu_A^\top\mu_Q)$. \\
$S(q,a;\theta)$ & Student-centric learning-cost score used to rank candidate answers. \\
$\lambda\in[0,1]$ & Trade-off parameter between $C_{AA}$ and $C_{AQ}$ in $S(q,a;\theta)$. \\
\midrule
\multicolumn{2}{@{}l}{\textbf{Part IV: Student-centric answer sampling}} \\
\midrule
\Ours{} & Student-Centric Answer Sampling, the proposed answer-selection and training procedure. \\
$R$ & Total number of selection-training rounds. \\
$r$ & Selection-training round index. \\
$M_q=|\mathcal{A}(q)|$ & Number of verified candidate answers for question $q$. \\
$a_{(k)}$ & The $k$-th candidate answer after sorting $\mathcal{A}(q)$ by increasing $S(q,a;\theta_r)$. \\
$G$ & Number of consecutive groups in \Ours{}; distinct from the gradient matrix $G$ in Table~\ref{tab:notA}. \\
$g$ & Group index, with $g=1,\dots,G$. \\
$\Pi_g(q)$ & The $g$-th group of sorted candidate answers for question $q$. \\
$\Pi_1(q)$ & Lowest-score group, from which \Ours{} samples during training. \\
$\lfloor\cdot\rfloor$ & Floor function used to define group boundaries. \\
$\tilde a(q)$ & Answer sampled for question $q$ from the lowest-score group. \\
$\mathrm{Unif}(\Pi_1(q))$ & Uniform distribution over candidates in $\Pi_1(q)$. \\
$\mathcal{D}_{\mathrm{low}}$ & Selected training set $\{(q,\tilde a(q)):q\in\mathcal{D}\}$ constructed from the lowest-score groups. \\
\midrule
\multicolumn{2}{@{}l}{\textbf{Part V: Complexity analysis}} \\
\midrule
$n_Q,n_A,n$ & Numbers of question tokens, answer tokens, and total tokens, with $n=n_Q+n_A$. \\
$F(n)$ & Time cost of one forward pass on a length-$n$ sequence. \\
$B(n)$ & Time cost of one backward pass on a length-$n$ sequence. \\
\bottomrule
\end{tabular}
\vspace{-2mm}
\label{tab:notB}
\end{table*}

\subsection{Datasets}
\label{abs:dataset}

\noindent \textbf{MATH}~\citep{Hendrycks2021MeasuringMP}. 
MATH is a benchmark of competition-style mathematical problems covering diverse subjects such as algebra, geometry, number theory, counting and probability, and precalculus. 
The problems require multi-step symbolic and quantitative reasoning, making the dataset a standard testbed for evaluating whether distilled students can learn from teacher-generated mathematical solutions. 
In our experiments, MATH is used as one of the mathematical reasoning sources for both answer-pool construction and evaluation.

\noindent \textbf{GSM8K}~\citep{Cobbe2021TrainingVT}. 
GSM8K contains grade-school math word problems paired with natural-language solutions and final numerical answers. 
Compared with MATH, GSM8K emphasizes everyday arithmetic reasoning expressed in textual form, and therefore provides a complementary test of whether selected teacher answers help students follow concise step-by-step problem solving. 
We use GSM8K as a mathematical reasoning source and evaluate student accuracy by checking the final answer.

\noindent \textbf{DeepScaleR}~\citep{deepscaler2025}. 
DeepScaleR is a reasoning-oriented dataset associated with the DeepScaleR training recipe, which studies scaling reinforcement learning for mathematical reasoning models. 
The dataset provides challenging mathematical prompts suitable for generating and verifying multiple teacher solutions. 
We use DeepScaleR as one of the mathematical reasoning sources to test whether \Ours{} remains effective on harder reasoning instances beyond standard math benchmarks.

\noindent \textbf{OpenR1-Math}~\citep{openr1update2025}. 
OpenR1-Math is a large-scale mathematical reasoning dataset released as part of the Open R1 effort. 
It aggregates and filters math problems for training reasoning models, and is intended to support open reproduction of reasoning-oriented post-training pipelines. 
In our experiments, OpenR1-Math provides an additional source of verifiable math questions for constructing candidate teacher answers and evaluating the resulting student models.

\noindent \textbf{Magpie}~\citep{Xu2024MagpieAD}. 
Magpie is an alignment-data synthesis framework that constructs instruction-following data by prompting aligned language models to produce instructions and responses. 
Unlike the mathematical datasets above, Magpie covers open-ended instruction-following scenarios where correctness is less easily captured by exact final-answer matching. 
We use Magpie as the training source for open-ended instruction tuning.

\noindent \textbf{IFEval}~\citep{Zhou2023InstructionFollowingEF}. 
IFEval evaluates instruction-following ability using prompts with explicitly checkable constraints, such as formatting requirements, keyword constraints, and structural conditions. 
Because many of its criteria can be verified automatically, it provides a controlled benchmark for measuring whether a student follows user instructions after fine-tuning. 
We use IFEval as one of the open-ended evaluation benchmarks.

\noindent \textbf{LiveBench}~\citep{White2024LiveBenchAC}. 
LiveBench is a contamination-limited benchmark designed to evaluate contemporary language models. 
It covers multiple ability categories and is intended to reduce the risk that benchmark items have appeared in model training data. 
We use LiveBench as a complementary open-ended evaluation benchmark for testing the generality of student-centric answer selection.

\subsection{Baseline Adaptation and Comparison Protocol}
\label{app:baseline-adaptation}

All methods operate on the same candidate pool \(\mathcal{A}(q)\) for each
question \(q\) and select one candidate answer per question to construct the
SFT data. The selected data are then used to train the same student model with
the same number of training examples and fine-tuning configuration. We
describe below how methods with different native granularities are adapted to
this common candidate-answer selection setting.

\noindent\textbf{RSR}~\citep{Yang2026WhichRT}.
RSR is applied as a teacher-level selection method. It selects a teacher, and
we use that teacher's verified candidate answer as supervision for each
question.

\noindent\textbf{GRACE}~\citep{Panigrahi2025InGG}.
GRACE selects the teacher best matched to the current student. For each
question, we use the verified candidate answer generated by the selected
teacher.

\noindent\textbf{RSR-Item}~\citep{Yang2026WhichRT}.
RSR-Item is the item-level variant of RSR. We apply it directly at the answer
level by scoring every candidate and selecting the preferred answer for each
question.

\noindent\textbf{PPL}~\citep{Bengio2003ANP}.
We compute perplexity for each question--answer pair and rank the candidate
answers accordingly. Because the preferred score direction is not fixed
a priori in this setting, we evaluate both low-PPL and high-PPL selection
variants.

\noindent\textbf{IFD}~\citep{Li2023FromQT}.
We compute instruction-following difficulty for each question--answer pair
and use it to rank the candidate answers. As with PPL, we evaluate both
low-score and high-score selection variants.

\noindent\textbf{Router}~\citep{Liu2026NoSB}.
We treat the candidate teacher answers as routing options. The router selects
one candidate answer for each question, which is then used as student
supervision.

\noindent\textbf{Influence}~\citep{Humane2025InfluenceFF}.
We estimate the influence of each candidate answer on the student and select
the candidate with the highest estimated utility for each question.

\section{Additional Results on Llama-3.2 Student Models}
\label{app:llama-results}

Table~\ref{tab:llama-overall-performance} reports the complete results for
Llama-3.2-1B-Instruct and Llama-3.2-3B-Instruct under the same experimental
protocol as Table~\ref{tab:overall-performance}. \Ours{} achieves the highest
mean on every benchmark and for both model-level averages. Compared with the
strongest competing average, \Ours{} improves performance by 4.18 points for
Llama-3.2-1B and 3.72 points for Llama-3.2-3B. These results demonstrate that
the effectiveness of student-centric answer selection generalizes beyond the
Qwen2.5 model family.

\begin{table*}[t]
\caption{
\textbf{Overall Performance of Answer-selection Methods on Llama-3.2 Student
Models.}
}
\label{tab:llama-overall-performance}
\vspace{-2mm}
\centering
\resizebox{\textwidth}{!}{
\begin{tabular}{@{}llccccccc@{}}
\toprule
\textbf{Model} & \textbf{Method}
& \textbf{DeepScaleR}
& \textbf{GSM8K}
& \textbf{MATH}
& \textbf{OpenR1-Math}
& \textbf{IFEval}
& \textbf{LiveBench}
& \textbf{Avg.} \\
\midrule

\multirow{8}{*}{Llama-3.2-1B}
& RSR
& 10.47 $\pm$ 0.45
& \textcolor{blue}{\textbf{52.16 $\pm$ 0.85}}
& 22.44 $\pm$ 0.26
& 8.47 $\pm$ 1.17
& 50.65 $\pm$ 0.44
& 31.60 $\pm$ 0.27
& \textcolor{blue}{\textbf{29.30}} \\

& GRACE
& 12.03 $\pm$ 0.51
& 49.53 $\pm$ 0.80
& 21.19 $\pm$ 0.38
& 9.20 $\pm$ 0.70
& 47.87 $\pm$ 0.39
& \textcolor{blue}{\textbf{31.70 $\pm$ 0.23}}
& 28.59 \\

& RSR-Item
& 10.90 $\pm$ 0.60
& 39.25 $\pm$ 1.10
& 25.18 $\pm$ 0.33
& \textcolor{blue}{\textbf{9.90 $\pm$ 0.20}}
& \textcolor{blue}{\textbf{52.60 $\pm$ 0.64}}
& 30.20 $\pm$ 0.23
& 28.01 \\

& PPL
& 8.07 $\pm$ 1.16
& 8.87 $\pm$ 0.60
& 11.68 $\pm$ 0.07
& 8.93 $\pm$ 0.25
& 52.31 $\pm$ 0.52
& 29.60 $\pm$ 0.29
& 19.91 \\

& IFD
& 11.17 $\pm$ 0.46
& 18.10 $\pm$ 1.03
& 21.14 $\pm$ 0.21
& 7.10 $\pm$ 0.56
& 50.09 $\pm$ 0.40
& 30.40 $\pm$ 0.16
& 23.00 \\

& Router
& \textcolor{blue}{\textbf{12.30 $\pm$ 1.47}}
& 12.64 $\pm$ 0.32
& 15.23 $\pm$ 0.28
& 8.70 $\pm$ 0.35
& 52.50 $\pm$ 0.88
& 29.00 $\pm$ 0.60
& 21.73 \\

& Influence
& 12.20 $\pm$ 0.36
& 35.26 $\pm$ 2.66
& \textcolor{blue}{\textbf{29.59 $\pm$ 0.15}}
& 9.26 $\pm$ 0.97
& 50.76 $\pm$ 0.72
& 31.23 $\pm$ 0.39
& 28.05 \\

\rowcolor{gray!12}
& \Ours{}
& \textcolor{red}{\textbf{14.07 $\pm$ 0.50}}
& \textcolor{red}{\textbf{55.57 $\pm$ 1.59}}
& \textcolor{red}{\textbf{32.62 $\pm$ 0.72}}
& \textcolor{red}{\textbf{11.83 $\pm$ 0.45}}
& \textcolor{red}{\textbf{53.87 $\pm$ 0.35}}
& \textcolor{red}{\textbf{32.90 $\pm$ 0.20}}
& \textcolor{red}{\textbf{33.48}} \\

\midrule

\multirow{8}{*}{Llama-3.2-3B}
& RSR
& 19.70 $\pm$ 0.78
& 77.16 $\pm$ 0.24
& 40.35 $\pm$ 0.49
& 18.20 $\pm$ 0.26
& \textcolor{blue}{\textbf{74.86 $\pm$ 0.33}}
& 40.50 $\pm$ 0.34
& 45.13 \\

& GRACE
& 23.80 $\pm$ 0.52
& \textcolor{blue}{\textbf{79.48 $\pm$ 1.21}}
& \textcolor{blue}{\textbf{44.39 $\pm$ 0.43}}
& \textcolor{blue}{\textbf{21.87 $\pm$ 0.51}}
& 74.27 $\pm$ 0.59
& 41.90 $\pm$ 0.54
& \textcolor{blue}{\textbf{47.62}} \\

& RSR-Item
& 21.77 $\pm$ 1.24
& 64.11 $\pm$ 1.48
& 38.89 $\pm$ 0.57
& 17.13 $\pm$ 0.72
& 73.86 $\pm$ 0.62
& 41.20 $\pm$ 0.31
& 42.83 \\

& PPL
& 17.83 $\pm$ 0.51
& 16.65 $\pm$ 0.87
& 18.29 $\pm$ 0.08
& 14.53 $\pm$ 0.31
& 72.16 $\pm$ 0.59
& \textcolor{blue}{\textbf{42.20 $\pm$ 0.53}}
& 30.28 \\

& IFD
& \textcolor{blue}{\textbf{26.80 $\pm$ 1.15}}
& 46.98 $\pm$ 1.13
& 29.05 $\pm$ 0.51
& 21.70 $\pm$ 0.87
& 73.38 $\pm$ 0.51
& 40.50 $\pm$ 0.26
& 39.74 \\

& Router
& 20.53 $\pm$ 0.68
& 24.39 $\pm$ 0.73
& 21.10 $\pm$ 0.59
& 16.57 $\pm$ 0.25
& 72.64 $\pm$ 0.67
& 40.60 $\pm$ 0.33
& 32.64 \\

& Influence
& 25.03 $\pm$ 0.65
& 49.84 $\pm$ 1.18
& 34.03 $\pm$ 0.58
& 21.54 $\pm$ 0.57
& 74.43 $\pm$ 0.80
& 41.30 $\pm$ 0.52
& 41.03 \\

\rowcolor{gray!12}
& \Ours{}
& \textcolor{red}{\textbf{27.37 $\pm$ 0.64}}
& \textcolor{red}{\textbf{82.45 $\pm$ 0.81}}
& \textcolor{red}{\textbf{52.47 $\pm$ 0.73}}
& \textcolor{red}{\textbf{25.37 $\pm$ 0.32}}
& \textcolor{red}{\textbf{76.57 $\pm$ 0.40}}
& \textcolor{red}{\textbf{43.80 $\pm$ 0.19}}
& \textcolor{red}{\textbf{51.34}} \\

\bottomrule
\end{tabular}
}
\vspace{-2mm}
\end{table*}

\section{Gradient Decomposition Details}
\label{app:grad-decomp}

For completeness, we derive Eqs.~\eqref{eq:grad}, \eqref{eq:grad-norm-base}, and \eqref{eq:grad-norm-pairwise-full}. Recall that $A=XW$ with $X\in\mathbb{R}^{n\times d}$, $W\in\mathbb{R}^{d\times m}$, $A\in\mathbb{R}^{n\times m}$, and that $\mathcal{L}=\mathcal{L}(A)$ is a scalar loss depending on $W$ only through $A$. We write $G=\frac{\partial \mathcal{L}}{\partial A}\in\mathbb{R}^{n\times m}$ and use the Frobenius inner product $\langle U,V\rangle_F=\mathrm{tr}(U^\top V)$.

Since $A=XW$ and $X$ is treated as constant when differentiating with respect to $W$, the differential of $A$ is
\begin{equation}
dA = X\,dW.
\end{equation}
By the chain rule,
\begin{equation}
\begin{aligned}
d\mathcal{L}
&=
\left\langle
\frac{\partial \mathcal{L}}{\partial A},\, dA
\right\rangle_F \\
&=
\mathrm{tr}\!\big(G^\top dA\big) \\
&=
\mathrm{tr}\!\big(G^\top X\,dW\big).
\end{aligned}
\end{equation}
Using the cyclic property of the trace, we obtain
\begin{equation}
\begin{aligned}
\mathrm{tr}\!\big(G^\top X\,dW\big)
&=
\mathrm{tr}\!\big((X^\top G)^\top dW\big) \\
&=
\big\langle X^\top G,\, dW \big\rangle_F .
\end{aligned}
\end{equation}
Comparing with the definition $d\mathcal{L}=\langle \nabla_W \mathcal{L}, dW\rangle_F$ yields Eq.~\eqref{eq:grad}:
\begin{equation}
\nabla_W\mathcal{L} = X^\top G.
\end{equation}
Eq.~\eqref{eq:grad-norm-base} then follows immediately:
\begin{equation}
\big\|\nabla_W\mathcal{L}\big\|_F^2 \;=\; \big\|X^\top G\big\|_F^2.
\end{equation}

To obtain Eq.~\eqref{eq:grad-norm-pairwise-full}, we expand the Frobenius norm in elementwise form. Let $z_i^\top$ and $g_i^\top$ denote the $i$-th rows of $X$ and $G$, respectively. Then $X X^\top$ and $G G^\top$ have entries $(X X^\top)_{ij}=z_i^\top z_j$ and $(G G^\top)_{ij}=g_i^\top g_j$. Using $\|M\|_F^2=\mathrm{tr}(M^\top M)$ and the trace identity $\mathrm{tr}(AB)=\sum_{i,j} A_{ij} B_{ji}$, we have
\begin{equation}
\begin{aligned}
\big\|X^\top G\big\|_F^2
&=
\mathrm{tr}\!\big(G^\top X X^\top G\big)
=
\mathrm{tr}\!\big(X X^\top G G^\top\big) \\
&=
\sum_{i,j=1}^n
(X X^\top)_{ij} (G G^\top)_{ji} \\
&=
\sum_{i,j=1}^n
(z_i^\top z_j)\,(g_i^\top g_j).
\end{aligned}
\end{equation}
which yields Eq.~\eqref{eq:grad-norm-pairwise-full}.

\section{Complexity Analysis}
\label{app:complexity}

Our analysis is for a single target student weight matrix $W\in\mathbb{R}^{d\times m}$. For a candidate sequence $x=[q;\,a]$, let $n_Q=|I_Q|$, $n_A=|I_A|$, and $n=n_Q+n_A$. Let $F(n)$ and $B(n)$ denote the time cost of one forward pass and one backward pass of the student on a length-$n$ sequence.

Exact blockwise gradient scoring implied by Eq.~\eqref{eq:block-decomp} requires both token representations and token-level backpropagated gradients, and therefore incurs one forward pass and one backward pass:
\begin{equation}
\label{eq:exact-base-cost-revise}
F(n)+B(n).
\end{equation}

A direct implementation of Eq.~\eqref{eq:block-decomp} enumerates all token pairs in the AA, AQ/QA, and QQ blocks. Since these blocks contain $n_A^2$, $n_A n_Q$, and $n_Q^2$ pairs, respectively, and each pairwise term requires $O(d+m)$ arithmetic, the additional overhead is
{\small
\begin{equation}
\label{eq:exact-naive-overhead-revise}
O\!\big((n_A^2+n_A n_Q+n_Q^2)(d+m)\big)
=
O\!\big(n^2(d+m)\big).
\end{equation}
}
Thus, under naive token-pair enumeration, the per-candidate score computation cost is
\begin{equation}
\label{eq:exact-per-candidate-naive-revise}
O\!\big(F(n)+B(n)+n^2(d+m)\big).
\end{equation}

The same exact score can be computed more efficiently in matrix form. Let $X_A\in\mathbb{R}^{n_A\times d}$ and $X_Q\in\mathbb{R}^{n_Q\times d}$ denote the answer-side and question-side rows of $X$, and let $G_A\in\mathbb{R}^{n_A\times m}$ and $G_Q\in\mathbb{R}^{n_Q\times m}$ denote the corresponding rows of $G$. Then Eq.~\eqref{eq:block-decomp} can be written as
{\small
\begin{equation}
\label{eq:block-matrix-form-revise}
\begin{aligned}
\big\|\nabla_W \mathcal L\big\|_F^2
&=
\|X_A^\top G_A\|_F^2
+
2\big\langle X_A^\top G_A,\; X_Q^\top G_Q\big\rangle_F \\
&\quad+
\|X_Q^\top G_Q\|_F^2 .
\end{aligned}
\end{equation}
}
Computing $X_A^\top G_A$ costs $O(n_A d m)$ and computing $X_Q^\top G_Q$ costs $O(n_Q d m)$; the remaining Frobenius norms and inner product cost $O(dm)$. Hence, the additional overhead beyond forward/backward is
\begin{equation}
\label{eq:exact-optimized-overhead-revise}
O(ndm),
\end{equation}
and the per-candidate score computation cost under the optimized exact implementation becomes
\begin{equation}
\label{eq:exact-per-candidate-optimized-revise}
O\!\big(F(n)+B(n)+ndm\big).
\end{equation}

We now analyze the cost of computing our proxy in Eq.~\eqref{eq:score-repeat}. This requires only one forward pass. After the forward pass, the representation aggregates in Eq.~\eqref{eq:mu-repeat} are computed by linear-time accumulation over token positions, costing $O(nd)$. The difficulty terms in Eq.~\eqref{eq:difficulty-repeat} are obtained by averaging token-level NLL values already produced by the same forward pass, which costs $O(n)$. The remaining score computation costs $O(d)$. Therefore, the per-candidate score computation cost of our proxy is
\begin{equation}
\label{eq:proxy-per-candidate-revise}
O\!\big(F(n)+nd\big).
\end{equation}

The key difference is therefore that exact blockwise scoring requires a backward pass, and even under an optimized implementation still incurs an additional $O(ndm)$ post-backward cost, whereas our proxy is forward-only and requires only linear-time post-forward aggregation. In particular, compared with the optimized exact implementation, our score replaces
\begin{equation}
O\!\big(F(n)+B(n)+ndm\big)
\end{equation}
with
\begin{equation}
O\!\big(F(n)+nd\big).
\end{equation}
This is the main computational advantage of our method for large-scale answer selection.

\section{Layer Analysis}
\label{subsec:layer-analysis}

\begin{figure}[t]
\centering
\includegraphics[width=0.95\linewidth]{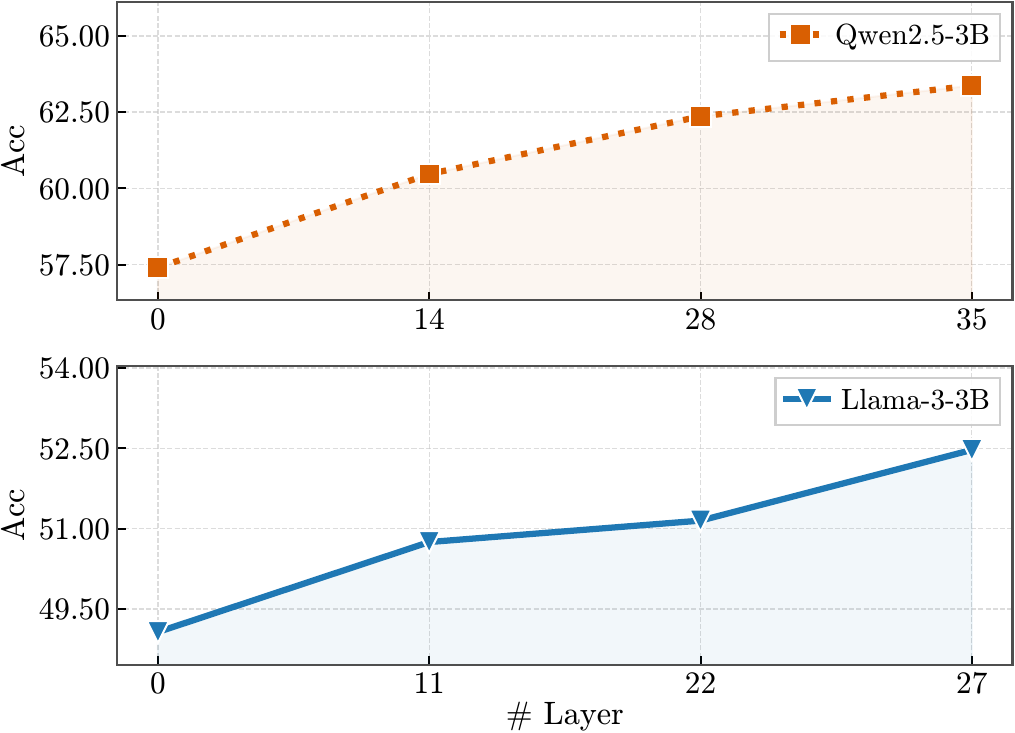}
\caption{
\textbf{Effect of the Representation Layer on Math.}
}
\label{fig:layer-analysis}
\end{figure}
In our main experiments, \Ours{} uses the final-layer hidden states of the student model to compute the answer-selection score. 
This choice follows our forward-only proxy: since token probabilities are computed from final-layer representations through the language modeling head, these representations provide the most direct connection to the NLL-based difficulty terms in our score.
To examine the effect of layer choice, we conduct a layer analysis on the Math setting with Llama-3.2-3B and Qwen2.5-3B. 
For each model, we select several approximately evenly spaced layers from shallow to deep layers, and additionally include the final layer. 
All other components of the training and selection pipeline are kept unchanged. 
The results are shown in Figure~\ref{fig:layer-analysis}.
As the extraction layer becomes deeper, \Ours{} achieves consistently better performance on both student models. 
The final layer obtains the highest average accuracy for both Llama-3.2-3B and Qwen2.5-3B. 
This result supports our use of final-layer hidden states, suggesting that representations closest to the language modeling head better capture the student-centric signal needed for answer selection.

\section{Scaling Strong Teachers by Answer Union}
\label{app:teacher-pool-scaling}

We add a controlled appendix experiment to test whether simply merging answers from strong teachers can replace student-centric answer selection.
We use the same nine teacher models and the same controlled teacher--student setup as Figure~\ref{fig:intro_image}; details of the teacher set, common-correct training subset, and fine-tuning protocol are provided in Appendix~\ref{app:intro-figure-detail}.
Within this setup, we rank the nine teachers by their direct test-set performance and construct cumulative teacher pools: top-1, top-1--2, top-1--3, and so on up to top-1--9.
For each cumulative pool, we use the corresponding common-correct training subset and take the union of the selected teachers' verified answers as the synthetic supervision set.
The student is fine-tuned on this answer union for three epochs.
There is no dynamic answer re-selection in this experiment.
This setup asks whether adding more high-performing teachers, while keeping the controlled protocol of Figure~\ref{fig:intro_image}, is sufficient to match \Ours{}.

\begin{figure*}[t]
\centering
\includegraphics[width=1.0\linewidth]{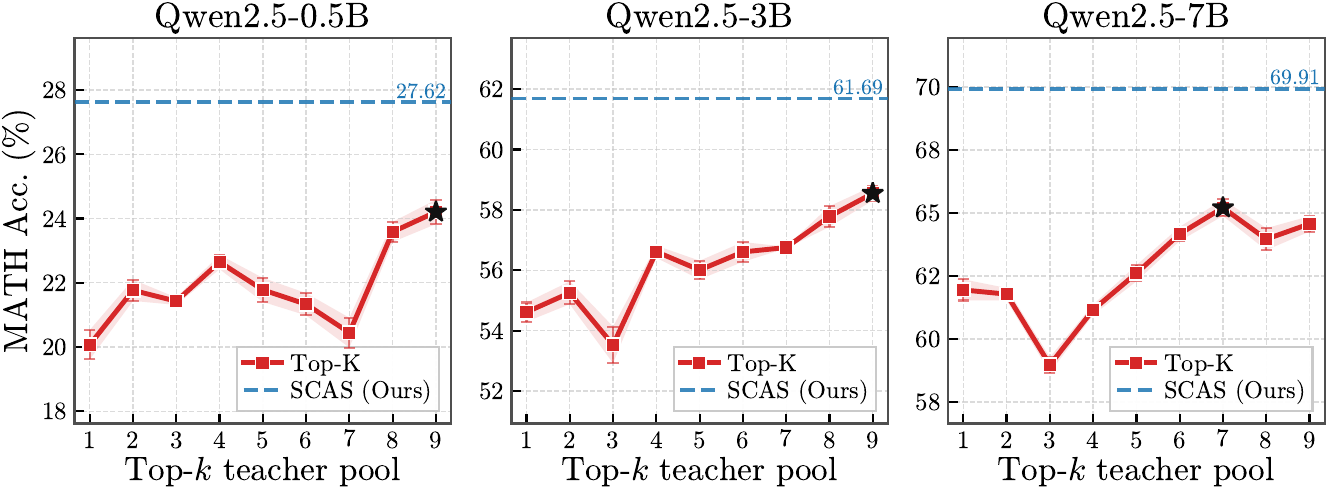}
\caption{
\textbf{Answer-Union Scaling with Strong Teachers.}
Top-$k$ merges $k$ teachers' verified answers, giving $k$ training answers per retained question; \Ours{} uses one selected answer per question.
The shaded band shows one standard deviation.
}
\label{fig:teacher-pool-scaling}
\end{figure*}

Figure~\ref{fig:teacher-pool-scaling} reports the MATH results for Qwen2.5 students.
Adding more high-performing teachers often improves over using only the single best teacher, but the gains are not enough to match \Ours{}.
For Qwen2.5-0.5B, the answer-union baseline increases from 20.07 at top-1 to 24.21 at top-9, still 3.41 points below \Ours{}.
For Qwen2.5-3B, it rises from 54.61 at top-1 to 58.55 at top-9, remaining 3.14 points below \Ours{}.
For Qwen2.5-7B, it peaks at top-7 with 65.21 and then declines, leaving a 4.70-point gap from \Ours{} even at its best point.
Importantly, the top-$k$ union baseline is not a fixed-budget comparison: it trains on $k$ verified answers for each retained question, so its answer-level training set is $k$ times larger than the one-answer-per-question supervision used by \Ours{} on the same retained questions.
Thus, even with up to $9\times$ answer-level training data, larger unions of strong-teacher answers do not recover the benefit of student-centric selection.
The result supports our central claim: what matters is not only having more correct answers from stronger teachers, but selecting the answers that better match the current student's learning state.

\section{Verification of Candidate Response Correctness and Quality}
\label{app:candidate-verification}

\noindent\textbf{Mathematical reasoning tasks.}
Candidate responses are not filtered using exact string matching. For each
teacher-generated response, an LLM verifier receives the problem, the
reference answer, and the complete candidate solution, and evaluates both the
final answer and the reasoning process. A candidate is retained only if its
final answer is mathematically equivalent to the reference answer and its
reasoning contains no substantive mathematical or logical error. Thus, a
response with a correct final answer but flawed reasoning is rejected, while
equivalent symbolic forms, simplified fractions, minor formatting differences,
and omitted routine steps are accepted.
To validate the model-based verification procedure, we randomly sampled 50
candidate solutions from each of GSM8K, MATH, DeepScaleR, and OpenR1-Math,
yielding 200 examples in total. Four annotators independently evaluated each
solution using the same final-answer and reasoning criteria, and the
majority-vote label was used as the human reference. As shown in
Table~\ref{tab:math-verifier-audit}, the LLM verifier agreed with the human
majority on 192 of 200 examples (96.0\%), with per-dataset agreement ranging
from 94.0\% to 98.0\%. These results support the reliability of the verifier
in handling mathematically equivalent answers while rejecting solutions with
substantive reasoning errors.

\noindent\textbf{Open-ended instruction tuning.}
Open-ended instructions generally do not admit a unique reference answer.
Therefore, we refer to the corresponding Magpie responses as
\emph{quality-audited candidate responses}, rather than verified-correct
answers. All candidates were generated by Llama-3.3-70B-Instruct under the
same decoding configuration. To assess their quality, we randomly sampled
1,000 responses and evaluated them using GPT-5 with a 1--5 rubric covering
instruction following, factual or logical correctness, relevance and
completeness, and clarity and coherence.
Although a score of 3 denotes a usable response with minor issues, we adopt a
more conservative threshold and count only scores of 4--5 as high quality.
GPT-5 rated 998 of the 1,000 sampled responses (99.8\%) as high quality. We
further conducted a human evaluation on a random 100-response subset. Four
annotators independently applied the same rubric, and the final human label
was determined by majority vote. Human judgments classified 98 responses as
high quality, while GPT-5 classified 99. Their binary judgments agreed on
99 of the 100 examples. These results support the overall quality of the
open-ended candidate pool and the reliability of the model-based audit.

\begin{table*}[t]
\caption{Prompt used to verify mathematical candidate solutions.}
\label{tab:prompt-math-verification}
\begin{evalpromptbox}{Evaluator prompt: Mathematical candidate verification}
\promptrole{System Prompt:}

You are an expert mathematics evaluator.

Given a math problem, a reference answer, and a candidate solution, determine
whether the candidate solution is correct.

Evaluate both the final answer and the reasoning process. A candidate should be
marked correct only if its final answer is mathematically equivalent to the
reference answer and its reasoning contains no substantive mathematical or
logical error. A candidate with a correct final answer but incorrect reasoning
must be marked incorrect.

Minor formatting differences or omitted routine steps should not be treated as
errors.

Respond with ONLY a valid JSON object:

\begin{quote}
\ttfamily\footnotesize
\{\\
\quad "correct": true,\\
\quad "model\_answer": "extracted final answer",\\
\quad "final\_answer\_correct": true,\\
\quad "reasoning\_correct": true,\\
\quad "reasoning": "brief explanation"\\
\}
\end{quote}

\promptseparator
\promptrole{User Prompt:}

\promptrole{Problem:}

\answerplaceholder{\textless problem\textgreater}

\promptrole{Reference Answer:}

\answerplaceholder{\textless reference answer\textgreater}

\promptrole{Candidate Solution:}

\answerplaceholder{\textless candidate solution\textgreater}

\promptseparator
\promptrole{Assistant Response:}
\answerplaceholder{\textless verifier JSON output\textgreater}
\end{evalpromptbox}
\end{table*}

\begin{table}[t]
\caption{Agreement between the mathematical-response verifier and
majority-vote human judgments.}
\centering
\small
\begin{tabular}{lcc}
\toprule
\textbf{Dataset} & \textbf{Samples} & \textbf{LLM--human agreement} \\
\midrule
GSM8K       & 50  & 49/50 (98.0\%) \\
MATH        & 50  & 48/50 (96.0\%) \\
DeepScaleR  & 50  & 47/50 (94.0\%) \\
OpenR1-Math & 50  & 48/50 (96.0\%) \\
\midrule
Overall     & 200 & 192/200 (96.0\%) \\
\bottomrule
\end{tabular}
\label{tab:math-verifier-audit}
\end{table}

\begin{table*}[t]
\caption{Prompt used for model-based quality audit of Magpie candidate
responses.}
\label{tab:prompt-magpie-quality-audit}
\begin{evalpromptbox}{Evaluator prompt: Magpie quality audit}
\promptrole{System Prompt:}

You are an expert evaluator of open-ended instruction-following responses.

Given a user instruction and a candidate response, evaluate the overall quality
of the candidate response according to the following criteria:

1. Instruction Following:

Whether the response correctly follows the user's explicit requirements and
intended goal.

2. Factual or Logical Correctness:

Whether the response is factually accurate and logically consistent. For
creative or subjective tasks, evaluate plausibility and internal consistency
rather than requiring a unique answer.

3. Relevance and Completeness:

Whether the response directly addresses the instruction and sufficiently
covers its important aspects.

4. Clarity and Coherence:

Whether the response is clear, understandable, and well organized.

Open-ended instructions may admit multiple valid responses and may not have a
unique reference answer. Do not penalize a response merely because it uses
different wording, organization, or another valid approach.

Use the following scoring scale:

5 = Excellent:

Fully satisfies all evaluation criteria and represents exemplary
instruction-tuning data.

4 = High quality:

Satisfies all important evaluation criteria, provides a good response, and is
suitable for instruction tuning.

3 = Good:

Suitable for instruction tuning, with minor issues that do not affect its core
correctness or usefulness.

2 = Weak:

Contains a notable error, omission, irrelevant content, or failure to satisfy
an important requirement, and is not suitable for instruction tuning.

1 = Unusable:

Contains serious factual or logical errors, is largely irrelevant or
incoherent, violates the main instruction, or contains unsafe content.

Both scores 4 and 5 indicate high-quality responses. A score of 5 should be
reserved for responses that are especially complete, precise, and well
presented.

Set ``usable\_for\_training'' to true for scores of 3, 4, or 5.
Set ``usable\_for\_training'' to false for scores of 1 or 2.

Respond with ONLY a valid JSON object in the following format:

\begin{quote}
\ttfamily\footnotesize
\{\\
\quad "score": 1,\\
\quad "usable\_for\_training": false,\\
\quad "issues": "none or a brief description of the identified issues",\\
\quad "reasoning": "a brief explanation of the evaluation"\\
\}
\end{quote}

Example:

\promptrole{Instruction:}

Write a polite two-sentence email declining a meeting because of a scheduling
conflict and suggest meeting next week instead.

\promptrole{Candidate Response:}

Thank you for the invitation, but I am unable to attend due to a scheduling
conflict. Would it be possible to reschedule the meeting for next week?

\promptrole{Evaluation:}

\begin{quote}
\ttfamily\footnotesize
\{\\
\quad "score": 5,\\
\quad "usable\_for\_training": true,\\
\quad "issues": "none",\\
\quad "reasoning": "The response follows all requested constraints and is
polite, relevant, complete, and clear."\\
\}
\end{quote}

\promptseparator
\promptrole{User Prompt:}

Evaluate the following response.

\promptrole{Instruction:}

\answerplaceholder{\textless user instruction\textgreater}

\promptrole{Candidate Response:}

\answerplaceholder{\textless candidate response\textgreater}

\promptrole{Evaluation:}

\promptseparator
\promptrole{Assistant Response:}
\answerplaceholder{\textless evaluator JSON output\textgreater}
\end{evalpromptbox}
\end{table*}

\begin{table*}[h]
\caption{Quality audit of teacher-generated Magpie responses. High and lower
denote scores of 4--5 and 1--3, respectively. The human evaluation set is a
subset of the 1,000-response model-based audit set.}
\centering
\small
\begin{tabular}{lccc}
\toprule
\textbf{Judgment source}
& \textbf{High}
& \textbf{Lower}
& \textbf{Human agreement} \\
\midrule
GPT-5 ($n=1{,}000$)
& 998 (99.8\%)
& 2 (0.2\%)
& -- \\

Human majority ($n=100$)
& 98 (98.0\%)
& 2 (2.0\%)
& -- \\

GPT-5 ($n=100$)
& 99 (99.0\%)
& 1 (1.0\%)
& 99/100 (99.0\%) \\
\bottomrule
\end{tabular}

\label{tab:magpie-quality-audit}
\end{table*}

\section{Additional Evaluation on Logical Reasoning}
\label{app:logical-reasoning}

To evaluate whether the
method generalizes beyond mathematical reasoning, we conduct an additional
experiment on logical reasoning.
We fine-tune Qwen2.5-7B-Instruct on 8,000 examples from
MiniMaxAI/SynLogic\footnote{\url{https://huggingface.co/datasets/MiniMaxAI/SynLogic}}.
The candidate-answer pool is generated by Gemma-4-31B, GPT-OSS-120B,
OLMo-3.1-32B, Phi-4-Reasoning, and Qwen3.5-27B. We evaluate the resulting
student on the Large and Medium subsets of
ZebraLogicBench\footnote{\url{https://huggingface.co/datasets/allenai/ZebraLogicBench}}.
All methods use the same candidate pool, student model, training-set size, and
fine-tuning configuration.
As shown in Table~\ref{tab:logical-reasoning}, \Ours{} outperforms the
strongest baseline by 2.1 points on the Large subset and 4.9 points on the
Medium subset. This experiment provides additional evidence that
student-centric answer selection generalizes from mathematical tasks to
logical reasoning.
\begin{table}[h]
\caption{Performance on ZebraLogicBench after fine-tuning
Qwen2.5-7B-Instruct on SynLogic. Values are reported as mean $\pm$ standard
deviation over three runs.}
\centering
\small
\begin{tabular}{lcc}
\toprule
\textbf{Method} & \textbf{Large} & \textbf{Medium} \\
\midrule
GRACE
& \textcolor{blue}{\textbf{7.2 $\pm$ 0.5}}
& 49.1 $\pm$ 0.8 \\

RSR
& 5.1 $\pm$ 0.4
& \textcolor{blue}{\textbf{51.3 $\pm$ 0.7}} \\

\rowcolor{gray!12}
\Ours{}
& \textcolor{red}{\textbf{9.3 $\pm$ 0.6}}
& \textcolor{red}{\textbf{56.2 $\pm$ 0.6}} \\
\bottomrule
\end{tabular}
\label{tab:logical-reasoning}
\end{table}

\section{Theoretical Analysis}
\label{app:theory}

In this appendix, we justify the proposed score $S(q,a;\theta)$ as a \emph{ranking-oriented forward proxy} for the answer-involving interactions induced by a candidate answer $a$ under the current student $\theta$.
Our goal is \emph{not} to reconstruct the full-model gradient norm exactly.
Instead, we derive a tractable answer-level quantity whose ordering is aligned with a retained last-layer interaction energy whenever the approximation error is smaller than the inter-candidate margin.
Consistent with the implemented score, our analysis focuses on final-layer hidden states and final LM-head interactions.

Our analysis proceeds in three stages.
First, we isolate the exact \emph{last-layer} answer-involving interactions for a fixed question--answer sequence.
Second, we derive a \emph{token-sensitive forward proxy} by replacing backward residual magnitudes with forward uncertainty signals.
Third, we show that the implemented score in the main text arises from a \emph{blockwise constant reduction} of this token-sensitive proxy, where the blockwise mean token NLL is the natural least-squares representative and also provides a denoising effect for answer-level ranking.

\subsection{Exact last-layer answer-involving interactions}
\label{app:theory-exact}

Let $x=[q;\,a]$ denote the training-templated sequence formed by concatenating the question $q$ and a verified correct candidate answer $a$.
Let $I=\{1,\dots,|x|\}$ be the token positions of $x$, and write
\[
I = I_Q \cup I_A,
\]
where $I_Q$ indexes conditioning/question tokens and $I_A$ indexes answer tokens.

Although the motivating decomposition in Section~\ref{subsec:formulation-motivation} is written for a generic weight matrix, the tractable forward proxy we derive here targets the \emph{final LM head}, where the output residuals are available from a single forward pass.
We use this last-layer quantity as the ranking target for answer selection.

Consider the final LM head
\[
W \in \mathbb{R}^{d\times V},
\]
where $V$ is the vocabulary size.
For each position $t\in I$, let $h_t\in\mathbb{R}^d$ be the final hidden state, let
\[
\begin{aligned}
p_t &:= p_\theta(\cdot \mid x_{<t}) \in \mathbb{R}^V, \\
y_t &:= x_t, \\
r_t &:= p_t - e_{y_t} \in \mathbb{R}^V .
\end{aligned}
\]
where $e_{y_t}$ is the one-hot vector of the target token.
The sequence loss is
\[
\mathcal L(x)=\sum_{t\in I} -\log p_\theta(y_t\mid x_{<t}).
\]

The gradient of the sequence loss with respect to the output head is
\begin{equation}
\label{eq:app-last-layer-grad}
\nabla_W \mathcal L(x)
=
\sum_{t\in I} h_t r_t^\top.
\end{equation}
Therefore,
\begin{equation}
\label{eq:app-last-layer-grad-norm}
\|\nabla_W \mathcal L(x)\|_F^2
=
\sum_{i,j\in I}
(h_i^\top h_j)(r_i^\top r_j).
\end{equation}

Splitting token indices into conditioning and answer positions yields
{\small
\begin{equation}
\label{eq:app-last-layer-block}
\|\nabla_W \mathcal L(x)\|_F^2
=
G_{AA}(x) + 2G_{AQ}(x) + G_{QQ}(x),
\end{equation}
}
where
\begin{align}
G_{AA}(x)
&=
\sum_{i,j\in I_A}
(h_i^\top h_j)(r_i^\top r_j), \\
G_{AQ}(x)
&=
\sum_{i\in I_A,\;j\in I_Q}
(h_i^\top h_j)(r_i^\top r_j),
\end{align}
and $G_{QQ}(x)$ is defined analogously.

Since answer selection compares multiple candidate answers for a \emph{fixed} question $q$, the terms involving answer tokens are the primary objects of interest.
At the last layer, the QQ block depends only on the conditioning prefix and is therefore invariant across candidate answers for the same $q$.
Thus, for ranking purposes, it suffices to focus on the retained answer-involving interactions.

As in the main text, we work with normalized hidden states
\[
u_t := \frac{h_t}{\|h_t\|_2}.
\]
This yields the normalized retained blocks
{\small
\begin{equation}
\label{eq:app-normalized-blocks}
\begin{aligned}
\widetilde G_{AA}(x)
&=
\frac{1}{|I_A|^2}
\sum_{i,j\in I_A}
(u_i^\top u_j)(r_i^\top r_j), \\
\widetilde G_{AQ}(x)
&=
\frac{1}{|I_A||I_Q|}
\sum_{i\in I_A,\;j\in I_Q}
(u_i^\top u_j)(r_i^\top r_j).
\end{aligned}
\end{equation}
}
We then define the retained answer-involving interaction energy
{\small
\begin{equation}
\label{eq:app-retained-energy}
\begin{aligned}
\widetilde{\mathcal G}_{\lambda^\star}(q,a;\theta)
&:=
(1-\lambda^\star)\,\widetilde G_{AA}(x)
+
\lambda^\star\,\widetilde G_{AQ}(x), \\
&\hspace{2em}
\lambda^\star\in[0,1].
\end{aligned}
\end{equation}
}
Our goal is to derive an efficient forward-only proxy for $\widetilde{\mathcal G}_{\lambda^\star}(q,a;\theta)$.

\subsection{From exact interactions to a token-sensitive forward proxy}
\label{app:theory-token-sensitive}

We first relate the residual magnitude $\|r_t\|_2$ to the token-level negative log-likelihood (NLL)
\[
\ell_t := -\log p_\theta(y_t\mid x_{<t}).
\]

\begin{lemma}[Residual magnitude is comparable to token NLL]
\label{lem:app-residual-nll}
Assume
\[
p_\theta(y_t\mid x_{<t}) \ge \tau > 0.
\]
Then
\begin{equation}
\label{eq:app-residual-nll-bound}
\tau\,\ell_t
\;\le\;
\|r_t\|_2
\;\le\;
\sqrt{2}\,\ell_t.
\end{equation}
\end{lemma}

\begin{proof}
Let
\[
q_t := 1-p_\theta(y_t\mid x_{<t}).
\]
Since
\[
r_t = p_t - e_{y_t},
\]
we have
\[
\begin{aligned}
\|r_t\|_2^2
&=
\big(1-p_\theta(y_t\mid x_{<t})\big)^2
+
\sum_{v\neq y_t} p_t(v)^2 \\
&=
q_t^2 + \sum_{v\neq y_t} p_t(v)^2 .
\end{aligned}
\]
Hence
\[
q_t \le \|r_t\|_2 \le \sqrt{2}\,q_t.
\]

Also,
\[
\ell_t = -\log p_\theta(y_t\mid x_{<t}) = -\log(1-q_t).
\]
Using $-\log(1-u)\ge u$ for $u\in[0,1)$ gives
\[
q_t \le \ell_t.
\]
Since $p_\theta(y_t\mid x_{<t})\ge \tau$, we have $1-q_t\ge \tau$, and by the mean value theorem,
\[
\ell_t = -\log(1-q_t)
\le
\frac{q_t}{\tau}.
\]
Therefore,
\[
\tau\,\ell_t \le q_t \le \|r_t\|_2
\]
and
\[
\|r_t\|_2 \le \sqrt{2}\,q_t \le \sqrt{2}\,\ell_t,
\]
which proves the claim.
\end{proof}

The exact retained blocks in Eq.~\eqref{eq:app-normalized-blocks} depend not only on the magnitudes $\|r_t\|_2$, but also on the \emph{pairwise alignment} between residual directions.
To make this dependence explicit, define
\[
\gamma_{ij}
:=
\frac{r_i^\top r_j}{\|r_i\|_2\,\|r_j\|_2}
\in[-1,1].
\]

We assume that within each retained block, these pairwise alignment coefficients are concentrated around a blockwise average value.

\begin{assumption}[Blockwise residual alignment concentration]
\label{assump:app-alignment}
There exist coefficients $\rho_{AA},\rho_{AQ}\in[0,1]$ and dispersion levels $\delta_{AA},\delta_{AQ}\ge 0$ such that
\[
\begin{aligned}
\gamma_{ij}
&=
\rho_{AA} + \Delta^{AA}_{ij}, \\
|\Delta^{AA}_{ij}|
&\le
\delta_{AA},
\qquad i,j\in I_A .
\end{aligned}
\]
and
\[
\begin{aligned}
\gamma_{ij}
&=
\rho_{AQ} + \Delta^{AQ}_{ij}, \\
|\Delta^{AQ}_{ij}|
&\le
\delta_{AQ},
\qquad i\in I_A,\; j\in I_Q .
\end{aligned}
\]
\end{assumption}

Under this assumption, the exact retained interactions admit a clean decomposition into a residual-magnitude-weighted forward proxy plus an explicit alignment error.

Define the residual-weighted summaries
\begin{equation}
\label{eq:app-residual-weighted-summaries}
\begin{aligned}
\bar w_A
&=
\frac{1}{|I_A|}
\sum_{i\in I_A}\|r_i\|_2\,u_i, \\
\bar w_Q
&=
\frac{1}{|I_Q|}
\sum_{j\in I_Q}\|r_j\|_2\,u_j .
\end{aligned}
\end{equation}
and the corresponding residual-weighted block scores
{\small
\begin{equation}
\label{eq:app-residual-weighted-proxy}
\widehat T_{AA}(x) := \|\bar w_A\|_2^2,
\qquad
\widehat T_{AQ}(x) := \bar w_A^\top \bar w_Q.
\end{equation}
}

\begin{proposition}[Alignment-controlled residual-weighted proxy]
\label{prop:app-alignment-controlled}
Under Assumption~\ref{assump:app-alignment},
\begin{equation}
\label{eq:app-aa-align-decomp}
\widetilde G_{AA}(x)
=
\rho_{AA}\,\widehat T_{AA}(x)
+
E_{AA}^{\mathrm{align}}(x),
\end{equation}
and
\begin{equation}
\label{eq:app-aq-align-decomp}
\widetilde G_{AQ}(x)
=
\rho_{AQ}\,\widehat T_{AQ}(x)
+
E_{AQ}^{\mathrm{align}}(x),
\end{equation}
where
{\small
\begin{equation}
\label{eq:app-align-errors-aa}
E_{AA}^{\mathrm{align}}(x)
=
\frac{1}{|I_A|^2}
\sum_{i,j\in I_A}
(u_i^\top u_j)\|r_i\|_2\|r_j\|_2\,\Delta_{ij}^{AA},
\end{equation}
}

{\small
\begin{equation}
\label{eq:app-align-errors-aq}
E_{AQ}^{\mathrm{align}}(x)
=
\frac{1}{|I_A||I_Q|}
\sum_{\substack{i\in I_A \\ j\in I_Q}}
(u_i^\top u_j)\|r_i\|_2\|r_j\|_2\,\Delta_{ij}^{AQ}.
\end{equation}
}
and satisfy
\begin{equation}
\label{eq:app-align-errors-bound-aa}
|E_{AA}^{\mathrm{align}}(x)|
\le
\delta_{AA}\,
\widehat T_{AA}^{\mathrm{abs}}(x),
\end{equation}
\begin{equation}
\label{eq:app-align-errors-bound-aq}
|E_{AQ}^{\mathrm{align}}(x)|
\le
\delta_{AQ}\,
\widehat T_{AQ}^{\mathrm{abs}}(x),
\end{equation}
with
\[
\widehat T_{AA}^{\mathrm{abs}}(x)
:=
\frac{1}{|I_A|^2}
\sum_{i,j\in I_A}
|u_i^\top u_j|\,\|r_i\|_2\|r_j\|_2,
\]
and
\[
\widehat T_{AQ}^{\mathrm{abs}}(x)
:=
\frac{1}{|I_A||I_Q|}
\sum_{i\in I_A,\;j\in I_Q}
|u_i^\top u_j|\,\|r_i\|_2\|r_j\|_2.
\]
\end{proposition}

\begin{proof}
Substituting
\[
r_i^\top r_j
=
\|r_i\|_2\,\|r_j\|_2\,\gamma_{ij}
\]
into Eq.~\eqref{eq:app-normalized-blocks} and using Assumption~\ref{assump:app-alignment} gives
{\small
\[
\widetilde G_{AA}(x)
=
\frac{1}{|I_A|^2}
\sum_{i,j\in I_A}
(u_i^\top u_j)\|r_i\|_2\|r_j\|_2(\rho_{AA}+\Delta_{ij}^{AA}).
\]
}
The term involving $\rho_{AA}$ is exactly
\[
\rho_{AA}
\left\|
\frac{1}{|I_A|}\sum_{i\in I_A}\|r_i\|_2\,u_i
\right\|_2^2
=
\rho_{AA}\,\widehat T_{AA}(x),
\]
and the remainder is $E_{AA}^{\mathrm{align}}(x)$.
The AQ case is identical.
The absolute-value bounds follow from $|\Delta_{ij}^{B}|\le \delta_B$ and the triangle inequality.
\end{proof}

Proposition~\ref{prop:app-alignment-controlled} isolates the first source of approximation error: \emph{alignment dispersion within the retained blocks}.
The remaining step is to replace residual magnitudes $\|r_t\|_2$ by forward uncertainty signals.
By Lemma~\ref{lem:app-residual-nll}, token NLL provides such a signal up to bounded positive distortion.

Motivated by this, define the token-sensitive forward summaries
\begin{equation}
\label{eq:app-token-weighted-summaries}
\bar v_A
=
\frac{1}{|I_A|}\sum_{i\in I_A}\ell_i\,u_i,
\qquad
\bar v_Q
=
\frac{1}{|I_Q|}\sum_{j\in I_Q}\ell_j\,u_j,
\end{equation}
and the corresponding token-sensitive forward proxies
\begin{equation}
\label{eq:app-token-sensitive-proxy}
T_{AA}(x) := \|\bar v_A\|_2^2,
\qquad
T_{AQ}(x) := \bar v_A^\top \bar v_Q.
\end{equation}

To keep the subsequent reduction explicit, define the distortion incurred by replacing residual magnitudes with token NLL by
\begin{equation}
\label{eq:app-nll-distortion-def}
\begin{aligned}
\varepsilon_{AA}^{\mathrm{nll}}(x)
&:=
\left|
\widehat T_{AA}(x)-\kappa_{AA}T_{AA}(x)
\right|, \\
\varepsilon_{AQ}^{\mathrm{nll}}(x)
&:=
\left|
\widehat T_{AQ}(x)-\kappa_{AQ}T_{AQ}(x)
\right| .
\end{aligned}
\end{equation}
for some positive blockwise scale coefficients $\kappa_{AA},\kappa_{AQ}>0$.
These coefficients absorb the bounded positive distortion between $\|r_t\|_2$ and $\ell_t$, while $\varepsilon_{AA}^{\mathrm{nll}}$ and $\varepsilon_{AQ}^{\mathrm{nll}}$ capture the residual discrepancy due to within-block variation of that distortion.

Thus, after the alignment-controlled reduction in Proposition~\ref{prop:app-alignment-controlled}, the natural forward-only object is the token-sensitive proxy $(T_{AA},T_{AQ})$.
We next show why the implemented score in the main text uses \emph{blockwise mean} token NLL rather than token-specific uncertainty weights.

\subsection{Blockwise constant reduction and why mean token NLL}
\label{app:theory-mean-field}

Recall the blockwise mean representations
\begin{equation}
\label{eq:app-block-means}
\mu_A
=
\frac{1}{|I_A|}\sum_{i\in I_A}u_i,
\qquad
\mu_Q
=
\frac{1}{|I_Q|}\sum_{j\in I_Q}u_j,
\end{equation}
and the blockwise mean NLLs
\begin{equation}
\label{eq:app-block-mean-nll}
d_A
=
\frac{1}{|I_A|}\sum_{i\in I_A}\ell_i,
\qquad
d_Q
=
\frac{1}{|I_Q|}\sum_{j\in I_Q}\ell_j.
\end{equation}

We first formalize why the blockwise mean is the natural answer-level summary of token uncertainty.

\begin{lemma}[Optimal blockwise constant approximation]
\label{lem:app-best-constant}
The blockwise means $d_A$ and $d_Q$ are the unique minimizers of the within-block squared approximation errors:
\[
\begin{aligned}
d_A
&=
\operatorname*{arg\,min}_{c\in\mathbb R}
\frac{1}{|I_A|}
\sum_{i\in I_A}(\ell_i-c)^2, \\
d_Q
&=
\operatorname*{arg\,min}_{c\in\mathbb R}
\frac{1}{|I_Q|}
\sum_{j\in I_Q}(\ell_j-c)^2 .
\end{aligned}
\]
\end{lemma}

\begin{proof}
For the answer block,
\[
f_A(c)
=
\frac{1}{|I_A|}\sum_{i\in I_A}(\ell_i-c)^2.
\]
Differentiating gives
\[
f_A'(c)
=
-\frac{2}{|I_A|}\sum_{i\in I_A}(\ell_i-c),
\]
so the unique stationary point satisfies
\[
c=\frac{1}{|I_A|}\sum_{i\in I_A}\ell_i=d_A.
\]
Since $f_A''(c)=2>0$, this stationary point is the unique minimizer.
The proof for $d_Q$ is identical.
\end{proof}

Lemma~\ref{lem:app-best-constant} shows that replacing token-level uncertainty weights by $d_A$ and $d_Q$ is not an arbitrary heuristic.
It is the least-squares optimal blockwise constant reduction of token-level uncertainty.
This reduction is also natural for answer selection, which operates at the \emph{answer level} rather than at the level of token-wise uncertainty allocation.
In addition, averaging suppresses idiosyncratic token-level spikes, making the resulting answer ranking more robust to rare tokens, tokenization artifacts, and other localized effects that may not be stable across semantically equivalent candidates.

Let
\begin{equation}
\label{eq:app-block-nll-variance}
\begin{aligned}
\sigma_A^2
&=
\frac{1}{|I_A|}
\sum_{i\in I_A}(\ell_i-d_A)^2, \\
\sigma_Q^2
&=
\frac{1}{|I_Q|}
\sum_{j\in I_Q}(\ell_j-d_Q)^2 .
\end{aligned}
\end{equation}
denote the within-block NLL variances.

We now define the blockwise quantities used in the main text:
\begin{equation}
\label{eq:app-implemented-blocks}
\begin{aligned}
C_{AA}(q,a;\theta)
&=
d_A^2\|\mu_A\|_2^2, \\
C_{AQ}(q,a;\theta)
&=
d_A d_Q\,(\mu_A^\top\mu_Q).
\end{aligned}
\end{equation}

The next lemma formalizes the approximation error incurred by reducing the token-sensitive proxy to these blockwise quantities.

\begin{lemma}[Mean-field reduction error]
\label{lem:app-mean-field}
We have
{\small
\begin{equation}
\label{eq:app-va-bound}
\begin{aligned}
\|\bar v_A - d_A\mu_A\|_2
&\le
\sigma_A, \\
\|\bar v_Q - d_Q\mu_Q\|_2
&\le
\sigma_Q.
\end{aligned}
\end{equation}
}
Consequently,
{\small
\begin{equation}
\label{eq:app-aa-error}
\left|
T_{AA}(x) - C_{AA}(q,a;\theta)
\right|
\le
2d_A\sigma_A + \sigma_A^2,
\end{equation}
}
and
{\small
\begin{equation}
\label{eq:app-aq-error}
\left|
T_{AQ}(x) - C_{AQ}(q,a;\theta)
\right|
\le
d_Q\sigma_A + d_A\sigma_Q + \sigma_A\sigma_Q.
\end{equation}
}
\end{lemma}

\begin{proof}
For the answer block,
\[
\bar v_A - d_A\mu_A
=
\frac{1}{|I_A|}\sum_{i\in I_A}(\ell_i-d_A)u_i.
\]
Using $\|u_i\|_2=1$ and Jensen's inequality,
\[
\begin{aligned}
\|\bar v_A - d_A\mu_A\|_2
&\le
\frac{1}{|I_A|}
\sum_{i\in I_A}|\ell_i-d_A| \\
&\le
\left(
\frac{1}{|I_A|}
\sum_{i\in I_A}(\ell_i-d_A)^2
\right)^{1/2}
=
\sigma_A .
\end{aligned}
\]
The bound for $\bar v_Q$ is identical.

For Eq.~\eqref{eq:app-aa-error}, write
\[
T_{AA}(x) = \|\bar v_A\|_2^2.
\]
Then
\[
\begin{aligned}
&\left|
\|\bar v_A\|_2^2 - d_A^2\|\mu_A\|_2^2
\right| \\
&=
\left|
(\bar v_A-d_A\mu_A)^\top(\bar v_A+d_A\mu_A)
\right| \\
&\le
\|\bar v_A-d_A\mu_A\|_2
\big(\|\bar v_A\|_2 + d_A\|\mu_A\|_2\big).
\end{aligned}
\]
Since $\|\mu_A\|_2\le 1$ and
\[
\|\bar v_A\|_2
\le
\|d_A\mu_A\|_2 + \|\bar v_A-d_A\mu_A\|_2
\le
d_A + \sigma_A,
\]
we obtain
\[
\begin{aligned}
\left|
T_{AA}(x) - C_{AA}(q,a;\theta)
\right|
&\le
\sigma_A(2d_A+\sigma_A) \\
&=
2d_A\sigma_A+\sigma_A^2 .
\end{aligned}
\]

For Eq.~\eqref{eq:app-aq-error}, add and subtract $d_A\mu_A^\top \bar v_Q$:
\[
\begin{aligned}
\left|
\bar v_A^\top \bar v_Q
-
d_A d_Q\,\mu_A^\top\mu_Q
\right|
&\le
\left|
(\bar v_A-d_A\mu_A)^\top \bar v_Q
\right| \\
&\quad+
\left|
d_A\mu_A^\top(\bar v_Q-d_Q\mu_Q)
\right| .
\end{aligned}
\]
Using $\|\mu_A\|_2\le 1$, $\|\bar v_Q\|_2\le d_Q+\sigma_Q$, and Eq.~\eqref{eq:app-va-bound},
\[
\begin{aligned}
&\left|
T_{AQ}(x) - C_{AQ}(q,a;\theta)
\right| \\
&\quad\le
\sigma_A(d_Q+\sigma_Q) + d_A\sigma_Q \\
&\quad=
d_Q\sigma_A + d_A\sigma_Q + \sigma_A\sigma_Q .
\end{aligned}
\]
This proves the claim.
\end{proof}

Lemma~\ref{lem:app-mean-field} shows that the implemented score is obtained by a controlled blockwise reduction of the token-sensitive forward proxy.
The error terms are governed by the within-block uncertainty dispersions $\sigma_A$ and $\sigma_Q$.
When token-level uncertainty is relatively homogeneous within each block, the sequence-level mean NLL provides a tight summary.
When it is heterogeneous, the reduction trades some fidelity for improved answer-level robustness.

\subsection{A combined proxy view and a ranking guarantee}
\label{app:theory-final}

We now combine the previous steps.

Let
\begin{equation}
\label{eq:app-mf-errors-def}
\begin{aligned}
\varepsilon_{AA}^{\mathrm{mf}}(q,a;\theta)
&:=
2d_A\sigma_A+\sigma_A^2, \\
\varepsilon_{AQ}^{\mathrm{mf}}(q,a;\theta)
&:=
d_Q\sigma_A + d_A\sigma_Q + \sigma_A\sigma_Q .
\end{aligned}
\end{equation}

\begin{theorem}[Blockwise reduced forward proxy]
\label{thm:app-combined}
Under Assumption~\ref{assump:app-alignment}, there exist nonnegative coefficients $\alpha,\beta\ge 0$ such that
{\small
\begin{equation}
\label{eq:app-combined-bound}
\begin{aligned}
&\left|
\widetilde{\mathcal G}_{\lambda^\star}(q,a;\theta)
-
\Big(
\alpha\,C_{AA}(q,a;\theta)
+
\beta\,C_{AQ}(q,a;\theta)
\Big)
\right| \\
&\quad\le
\varepsilon_{\mathrm{tot}}(q,a;\theta),
\end{aligned}
\end{equation}
}
where
{\small
\begin{align}
\label{eq:app-total-error}
\varepsilon_{\mathrm{tot}}(q,a;\theta)
&=
(1-\lambda^\star)
\Big(
\varepsilon_{AA}^{\mathrm{align}}(x)
+
|\rho_{AA}|\,\varepsilon_{AA}^{\mathrm{nll}}(x)
\nonumber\\
&\qquad\qquad\qquad
+
|\rho_{AA}|\,\kappa_{AA}\,
\varepsilon_{AA}^{\mathrm{mf}}(q,a;\theta)
\Big)
\nonumber\\
&\quad+
\lambda^\star
\Big(
\varepsilon_{AQ}^{\mathrm{align}}(x)
+
|\rho_{AQ}|\,\varepsilon_{AQ}^{\mathrm{nll}}(x)
\nonumber\\
&\qquad\qquad\qquad
+
|\rho_{AQ}|\,\kappa_{AQ}\,
\varepsilon_{AQ}^{\mathrm{mf}}(q,a;\theta)
\Big),
\end{align}
}
with
{\small
\[
\begin{aligned}
\varepsilon_{AA}^{\mathrm{align}}(x)
&:=
|E_{AA}^{\mathrm{align}}(x)|, \\
\varepsilon_{AQ}^{\mathrm{align}}(x)
&:=
|E_{AQ}^{\mathrm{align}}(x)|.
\end{aligned}
\]
}
\end{theorem}

\begin{proof}
From Proposition~\ref{prop:app-alignment-controlled},
\[
\begin{aligned}
\widetilde G_{AA}(x)
&=
\rho_{AA}\widehat T_{AA}(x)
+
E_{AA}^{\mathrm{align}}(x), \\
\widetilde G_{AQ}(x)
&=
\rho_{AQ}\widehat T_{AQ}(x)
+
E_{AQ}^{\mathrm{align}}(x).
\end{aligned}
\]
By the definition of $\varepsilon_{AA}^{\mathrm{nll}}$ and $\varepsilon_{AQ}^{\mathrm{nll}}$ in Eq.~\eqref{eq:app-nll-distortion-def},
\[
\begin{aligned}
\widehat T_{AA}(x)
&=
\kappa_{AA}T_{AA}(x)
+
\Delta_{AA}^{\mathrm{nll}}(x), \\
|\Delta_{AA}^{\mathrm{nll}}(x)|
&\le
\varepsilon_{AA}^{\mathrm{nll}}(x).
\end{aligned}
\]
and analogously for the AQ block.
Finally, Lemma~\ref{lem:app-mean-field} gives
\[
\begin{aligned}
T_{AA}(x)
&=
C_{AA}(q,a;\theta)
+
\Delta_{AA}^{\mathrm{mf}}(q,a;\theta), \\
|\Delta_{AA}^{\mathrm{mf}}(q,a;\theta)|
&\le
\varepsilon_{AA}^{\mathrm{mf}}(q,a;\theta).
\end{aligned}
\]
and analogously for the AQ block.
Substituting these relations into Eq.~\eqref{eq:app-retained-energy} and collecting terms yields Eq.~\eqref{eq:app-combined-bound}, with
{\small
\[
\alpha=(1-\lambda^\star)\rho_{AA}\kappa_{AA},
\qquad
\beta=\lambda^\star\rho_{AQ}\kappa_{AQ}.
\]
}
\end{proof}

Theorem~\ref{thm:app-combined} gives the precise interpretation of the main-text score:
it is a \emph{blockwise reduced forward proxy} for the retained last-layer answer-involving interaction energy.
The total approximation error decomposes into three sources:
(i) \emph{alignment error}, arising from residual-direction dispersion within each retained block;
(ii) \emph{NLL-scale distortion}, arising from replacing residual magnitudes by token NLL;
and (iii) \emph{mean-field reduction error}, arising from replacing token-level uncertainty allocation by its optimal blockwise constant approximation.

Because $\alpha$ and $\beta$ are nonnegative, whenever $\alpha+\beta>0$ the proxy in Eq.~\eqref{eq:app-combined-bound} can be rescaled into the same nonnegative-mixture form as the implemented score in the main text:
\[
\begin{aligned}
S(q,a;\theta)
&=
(1-\lambda)\,C_{AA}(q,a;\theta)
+
\lambda\,C_{AQ}(q,a;\theta), \\
&\hspace{2em}
\lambda\in[0,1].
\end{aligned}
\]
In practice, $\lambda$ absorbs the relative blockwise scale between the AA and AQ terms.

The relevance of this proxy to answer selection is captured by the following ranking corollary.

\begin{corollary}[Margin-based ranking preservation]
\label{cor:app-ranking}
Fix a question $q$ and two verified candidate answers $a,b\in\mathcal A(q)$.
Suppose
\begin{equation}
\label{eq:app-ranking-margin}
\begin{aligned}
&\widetilde{\mathcal G}_{\lambda^\star}(q,a;\theta)
+
\varepsilon_{\mathrm{tot}}(q,a;\theta) \\
&\quad<
\widetilde{\mathcal G}_{\lambda^\star}(q,b;\theta)
-
\varepsilon_{\mathrm{tot}}(q,b;\theta).
\end{aligned}
\end{equation}
Then
\[
\begin{aligned}
&\alpha\,C_{AA}(q,a;\theta)
+
\beta\,C_{AQ}(q,a;\theta) \\
&\quad<
\alpha\,C_{AA}(q,b;\theta)
+
\beta\,C_{AQ}(q,b;\theta).
\end{aligned}
\]
Equivalently, whenever $\alpha+\beta>0$, after normalizing the nonnegative coefficients $\alpha$ and $\beta$, the implemented blockwise proxy preserves the ordering of $a$ and $b$.
\end{corollary}

\begin{proof}
Eq.~\eqref{eq:app-combined-bound} implies
\[
\begin{aligned}
&\alpha\,C_{AA}(q,a;\theta)
+
\beta\,C_{AQ}(q,a;\theta) \\
&\quad\le
\widetilde{\mathcal G}_{\lambda^\star}(q,a;\theta)
+
\varepsilon_{\mathrm{tot}}(q,a;\theta),
\end{aligned}
\]
and
\[
\begin{aligned}
&\alpha\,C_{AA}(q,b;\theta)
+
\beta\,C_{AQ}(q,b;\theta) \\
&\quad\ge
\widetilde{\mathcal G}_{\lambda^\star}(q,b;\theta)
-
\varepsilon_{\mathrm{tot}}(q,b;\theta).
\end{aligned}
\]
The claim follows immediately from Eq.~\eqref{eq:app-ranking-margin}.
\end{proof}

Corollary~\ref{cor:app-ranking} makes explicit why the proposed score is suitable for answer selection.
The score need not numerically match the exact retained interaction energy; it is sufficient that the approximation error remain smaller than the inter-candidate margin.
This is precisely the regime relevant to \Ours{}, whose objective is to preserve the relative ordering of candidate answers rather than to estimate the exact interaction value itself.

\section{Teacher--Student Family Compatibility}
\label{app:family-compatibility}

We investigate whether the gradient-based scores capture compatibility between
teacher and student model families. We conduct a controlled analysis using nine
teachers: Gemma-4-31B, GPT-5-Chat, GPT-OSS-120B, Llama-3.3-70B,
OLMo-3.1-32B, Phi-4-Reasoning, Qwen2.5-72B, Qwen3-32B, and Qwen3.5-27B.
The students are Qwen2.5-3B and Llama-3.2-3B. For the Qwen student, the three
Qwen teachers are treated as same-family and the remaining six as
cross-family. For the Llama student, Llama-3.3-70B is treated as same-family
and the remaining eight as cross-family.
We follow the controlled setup of Figure~\ref{fig:intro_image}. For each of
DeepScaleR and MATH, we sample 1,000 questions for which all nine teachers
produce verified-correct answers. We rank the nine candidate answers using
either the gradient score or the student-centric learning-cost score and
measure how often the selected answer is produced by a same-family teacher.
We report results for direct minimum-score selection and for the grouped
selection procedure used by \Ours{}. In the grouped setting, candidates are
partitioned using the same configuration as the main experiments, and one
answer is sampled uniformly from the lowest-score group. Applying this
procedure to the student-centric learning-cost score corresponds to \Ours{}.
Under family-agnostic uniform selection, the expected same-family rates are
\(3/9=33.3\%\) for Qwen2.5-3B and \(1/9=11.1\%\) for Llama-3.2-3B. Direct
minimum-gradient selection exceeds these reference rates on both datasets,
indicating that teacher--student family compatibility is reflected in the
gradient analysis, although the strength of this effect depends on the
student. Direct minimum learning-cost selection exhibits the same tendency.
Under the full \Ours{} grouping-and-sampling procedure, however, cross-family
answers still account for 64.0\%--67.3\% of the selected responses. Grouping
therefore reduces concentration on a single teacher family and preserves
substantial cross-family diversity. These results suggest that model family
influences candidate ranking but does not determine teacher selection under
\Ours{}. Because model family is only a coarse proxy for similarity between
training-data distributions, this analysis should be interpreted as evidence
of compatibility rather than a causal attribution to training-data overlap.

\begin{table*}[t]
\caption{Same-family selection rates under direct minimum-score selection and
grouped sampling. Under family-agnostic uniform selection, the expected rates
are 33.3\% for Qwen2.5-3B and 11.1\% for Llama-3.2-3B.}
\centering
\small
\begin{tabular}{llcccc}
\toprule
\textbf{Dataset}
& \textbf{Student}
& \textbf{Min.\ Gradient}
& \textbf{Min.\ Learning Cost}
& \textbf{Grouped Gradient}
& \textbf{\Ours{}} \\
\midrule
DeepScaleR
& Qwen2.5-3B
& 93.1\%
& 99.2\%
& 57.1\%
& 33.6\% \\

DeepScaleR
& Llama-3.2-3B
& 26.4\%
& 84.5\%
& 23.1\%
& 32.7\% \\

MATH
& Qwen2.5-3B
& 93.1\%
& 98.6\%
& 61.9\%
& 34.0\% \\

MATH
& Llama-3.2-3B
& 26.8\%
& 92.0\%
& 24.0\%
& 36.0\% \\
\bottomrule
\end{tabular}
\label{tab:family-compatibility}
\end{table*}

\section{Prompt Templates for Teacher Generation and Evaluation}
\label{app:prompt-templates}

To make the data construction and evaluation protocol reproducible, we report the prompts used in our experiments.
In each prompt, the dataset-specific question is inserted into the Question field, and the assistant field is reserved for model generation.
The output field denotes the generated response.
Tables~\ref{tab:prompt-teacher-deepscaler}, \ref{tab:prompt-teacher-gsm8k}, \ref{tab:prompt-teacher-math}, \ref{tab:prompt-teacher-magpie}, and \ref{tab:prompt-teacher-openr1} show the prompts used to generate teacher answers for DeepScaleR, GSM8K, MATH, Magpie, and OpenR1-Math.
Tables~\ref{tab:prompt-eval-deepscaler}, \ref{tab:prompt-eval-gsm8k}, \ref{tab:prompt-eval-math}, and \ref{tab:prompt-eval-openr1} show the prompts used for model-based evaluation.
The model outputs shown in the evaluation prompt examples are responses generated by Qwen2.5-7B-Instruct.

\begin{table*}[t]
\caption{Prompt used to generate teacher answers for DeepScaleR.}
\label{tab:prompt-teacher-deepscaler}
\begin{teacherpromptbox}{Teacher generation prompt: DeepScaleR}
\promptrole{System Prompt:} You are a helpful and precise math assistant.
\promptseparator
\promptrole{User Prompt:}

Please answer the math question below exactly as asked.

When solving the problem, provide a detailed step-by-step solution. Clearly explain each step, including the reasoning, formulas, substitutions, simplifications, and calculations used.

Do not give unnecessary unrelated explanations. Do not change the problem.
\promptseparator
\promptrole{Question:}

In triangle $A B C$ with altitude $A D, \angle B A C=45^{\circ}, D B=3$, and $C D=2$. Find the area of triangle $A B C$.
\promptseparator
\promptrole{Assistant Response:} \answerplaceholder{\textless teacher-generated answer\textgreater}
\end{teacherpromptbox}
\end{table*}

\begin{table*}[t]
\caption{Prompt used to generate teacher answers for GSM8K.}
\label{tab:prompt-teacher-gsm8k}
\begin{teacherpromptbox}{Teacher generation prompt: GSM8K}
\promptrole{System Prompt:} You are a helpful and precise math assistant.
\promptseparator
\promptrole{User Prompt:}

Please answer the math question below exactly as asked.

When solving the problem, provide a detailed step-by-step solution. Clearly explain each step, including the reasoning, formulas, substitutions, simplifications, and calculations used.

Do not give unnecessary unrelated explanations. Do not change the problem.
\promptseparator
\promptrole{Question:}

Fred had 236 dollars to spend on 6 books. After buying them, he had 14 dollars. On average, how much did each book cost?
\promptseparator
\promptrole{Assistant Response:} \answerplaceholder{\textless teacher-generated answer\textgreater}
\end{teacherpromptbox}
\end{table*}

\begin{table*}[t]
\caption{Prompt used to generate teacher answers for MATH.}
\label{tab:prompt-teacher-math}
\begin{teacherpromptbox}{Teacher generation prompt: MATH}
\promptrole{System Prompt:} You are a helpful and precise math assistant.
\promptseparator
\promptrole{User Prompt:}

Please answer the math question below exactly as asked.

When solving the problem, provide a detailed step-by-step solution. Clearly explain each step, including the reasoning, formulas, substitutions, simplifications, and calculations used.

Do not give unnecessary unrelated explanations. Do not change the problem.
\promptseparator
\promptrole{Question:}

In a rectangular coordinate system, what is the number of units in the distance from the origin to the point $(-15, 8)$?
\promptseparator
\promptrole{Assistant Response:} \answerplaceholder{\textless teacher-generated answer\textgreater}
\end{teacherpromptbox}
\end{table*}

\begin{table*}[t]
\caption{Prompt used to generate candidate responses for Magpie.}
\label{tab:prompt-teacher-magpie}
\begin{teacherpromptbox}{Teacher generation prompt: Magpie}
\promptrole{System Prompt:} You are a helpful and precise assistant.
\promptseparator
\promptrole{User Prompt:}

Please answer the question or instruction below exactly as asked.

When reasoning is needed, provide a clear step-by-step solution and explain each step.

Do not give unnecessary unrelated explanations. Do not change the question or instruction.
\promptseparator
\promptrole{Question or Instruction:}

What is the probability that a randomly selected integer from the set $\{10, 12, 15, 18, 20\}$ is divisible by both 2 and 3?
\promptseparator
\promptrole{Assistant Response:} \answerplaceholder{\textless teacher-generated response\textgreater}
\end{teacherpromptbox}
\end{table*}

\begin{table*}[t]
\caption{Prompt used to generate teacher answers for OpenR1-Math.}
\label{tab:prompt-teacher-openr1}
\begin{teacherpromptbox}{Teacher generation prompt: OpenR1-Math}
\promptrole{System Prompt:} You are a helpful and precise math assistant.
\promptseparator
\promptrole{User Prompt:}

Please answer the math question below exactly as asked.

When solving the problem, provide a detailed step-by-step solution. Clearly explain each step, including the reasoning, formulas, substitutions, simplifications, and calculations used.

Do not give unnecessary unrelated explanations. Do not change the problem.
\promptseparator
\promptrole{Question:}

I4.1 Given that $x$ and $y$ are real numbers such that $|x|+x+y=10$ and $|y|+x-y=10$. If $P=x+y$, find the value of $P$.
\promptseparator
\promptrole{Assistant Response:} \answerplaceholder{\textless teacher-generated answer\textgreater}
\end{teacherpromptbox}
\end{table*}

\begin{table*}[t]
\caption{Evaluator prompt used for DeepScaleR.}
\label{tab:prompt-eval-deepscaler}
\begin{evalpromptbox}{Evaluator prompt: DeepScaleR}
\promptrole{System Prompt:}

You are an expert mathematics evaluator. Your task is to evaluate whether a model's answer to a math problem is correct.

You will be given:
- The problem statement
- The reference answer
- The model's output

Carefully analyze the model's reasoning and final answer. Extract the answer from the model's output and compare it with the reference answer.

Respond with ONLY a JSON object (no other text):
\begin{quote}
\ttfamily\footnotesize
\{\\
\quad "correct": true or false,\\
\quad "model\_answer": "extracted answer",\\
\quad "reference\_answer": "reference answer",\\
\quad "reasoning": "brief explanation"\\
\}
\end{quote}
\promptseparator
\promptrole{User Prompt:}

\promptrole{Problem:}

$\tan 2\alpha = \frac{\cos \alpha}{2-\sin \alpha}$, where $0 < \alpha < \frac{\pi}{2}$, find the value of $\tan \alpha$.
\promptseparator
\promptrole{Reference Answer:}

\(\frac{\sqrt{15}}{15}\)
\promptseparator
\promptrole{Model Output:}

We are given:
\[
\tan 2\alpha = \frac{\cos \alpha}{2-\sin \alpha}, \quad 0 < \alpha < \frac{\pi}{2}.
\]
Use the double-angle identity
\[
\tan 2\alpha = \frac{2\tan \alpha}{1-\tan^2 \alpha}.
\]
Let \(t=\tan\alpha\). Expressing the right-hand side in terms of \(t\) gives
\[
\frac{\cos \alpha}{2-\sin \alpha}
=
\frac{1}{2\sqrt{1+t^2}-t}.
\]
Thus,
\[
\frac{2t}{1-t^2}=\frac{1}{2\sqrt{1+t^2}-t}.
\]
Cross-multiplying and simplifying gives
\[
4t\sqrt{1+t^2}=t^2+1.
\]
Squaring both sides,
\[
16t^2(1+t^2)=(t^2+1)^2,
\]
so
\[
15t^4+14t^2-1=0.
\]
Let \(u=t^2\). Then
\[
15u^2+14u-1=0,
\]
which yields \(u=\frac{1}{15}\) or \(u=-1\). Since \(u=t^2>0\), \(u=\frac{1}{15}\). Hence
\[
t=\tan\alpha=\frac{1}{\sqrt{15}}.
\]
Answer:
\[
\boxed{\tan\alpha=\frac{1}{\sqrt{15}}}.
\]
\promptseparator
\promptrole{Evaluation instruction:}
Evaluate if the model's answer is correct. Respond with JSON only.
\promptseparator
\promptrole{Assistant Response:} \answerplaceholder{\textless evaluator JSON output\textgreater}
\end{evalpromptbox}
\end{table*}

\begin{table*}[t]
\caption{Evaluator prompt used for GSM8K.}
\label{tab:prompt-eval-gsm8k}
\begin{evalpromptbox}{Evaluator prompt: GSM8K}
\promptrole{System Prompt:}

You are an expert mathematics evaluator. Your task is to evaluate whether a model's answer to a math problem is correct.

You will be given:
- The problem statement
- The reference answer
- The model's output

Carefully analyze the model's reasoning and final answer. Extract the answer from the model's output and compare it with the reference answer.

Respond with ONLY a JSON object (no other text):
\begin{quote}
\ttfamily\footnotesize
\{\\
\quad "correct": true or false,\\
\quad "model\_answer": "extracted answer",\\
\quad "reference\_answer": "reference answer",\\
\quad "reasoning": "brief explanation"\\
\}
\end{quote}
\promptseparator
\promptrole{User Prompt:}

\promptrole{Problem:}

James decides to run 3 sprints 3 times a week. He runs 60 meters each sprint. How many total meters does he run a week?
\promptseparator
\promptrole{Reference Answer:}

He sprints \(3*3=9\) times.
So he runs \(9*60=540\) meters.
\#\#\#\# 540
\promptseparator
\promptrole{Model Output:}

To find out how many total meters James runs in a week, we need to calculate the distance he covers in one session and then multiply it by the number of sessions he has in a week.

James runs 3 sprints each session, and each sprint is 60 meters. So, the distance he covers in one session is:
\[
3 \text{ sprints/session} \times 60 \text{ meters/sprint} = 180 \text{ meters/session}.
\]
Since he runs 3 times a week, we multiply the distance covered in one session by 3 to get the total distance for the week:
\[
180 \text{ meters/session} \times 3 \text{ sessions/week} = 540 \text{ meters/week}.
\]
Therefore, James runs a total of 540 meters a week.
\promptseparator
\promptrole{Evaluation instruction:}
Evaluate if the model's answer is correct. Respond with JSON only.
\promptseparator
\promptrole{Assistant Response:} \answerplaceholder{\textless evaluator JSON output\textgreater}
\end{evalpromptbox}
\end{table*}

\begin{table*}[t]
\caption{Evaluator prompt used for MATH.}
\label{tab:prompt-eval-math}
\begin{evalpromptbox}{Evaluator prompt: MATH}
\promptrole{System Prompt:}

You are an expert mathematics evaluator. Your task is to evaluate whether a model's answer to a math problem is correct.

You will be given:
- The problem statement
- The reference answer
- The model's output

Carefully analyze the model's reasoning and final answer. Extract the answer from the model's output and compare it with the reference answer.

Respond with ONLY a JSON object (no other text):
\begin{quote}
\ttfamily\footnotesize
\{\\
\quad "correct": true or false,\\
\quad "model\_answer": "extracted answer",\\
\quad "reference\_answer": "reference answer",\\
\quad "reasoning": "brief explanation"\\
\}
\end{quote}
\promptseparator
\promptrole{User Prompt:}

\promptrole{Problem:}

What is the sum of the four positive factors of the positive integer value of \(\sqrt{196}\)?
\promptseparator
\promptrole{Reference Answer:}

Calculate \(\sqrt{196}=\sqrt{2^2\cdot7^2}=2\cdot7\). The sum of the four positive factors is \(1+2+7+14=\boxed{24}\).
\promptseparator
\promptrole{Model Output:}

First, find the positive integer value of \(\sqrt{196}\):
\[
\sqrt{196}=14.
\]
Next, determine the positive factors of 14:

The factors of 14 are \(1, 2, 7,\) and \(14\).

Sum these factors:
\[
1+2+7+14=24.
\]

Answer:
\[
\boxed{24}
\]
\promptseparator
\promptrole{Evaluation instruction:}
Evaluate if the model's answer is correct. Respond with JSON only.
\promptseparator
\promptrole{Assistant Response:} \answerplaceholder{\textless evaluator JSON output\textgreater}
\end{evalpromptbox}
\end{table*}

\begin{table*}[t]
\caption{Evaluator prompt used for OpenR1-Math.}
\label{tab:prompt-eval-openr1}
\begin{evalpromptbox}{Evaluator prompt: OpenR1-Math}
\promptrole{System Prompt:}

You are an expert mathematics evaluator. Your task is to evaluate whether a model's answer to a math problem is correct.

You will be given:
- The problem statement
- The reference answer
- The model's output

Carefully analyze the model's reasoning and final answer. Extract the answer from the model's output and compare it with the reference answer.

Respond with ONLY a JSON object (no other text):
\begin{quote}
\ttfamily\footnotesize
\{\\
\quad "correct": true or false,\\
\quad "model\_answer": "extracted answer",\\
\quad "reference\_answer": "reference answer",\\
\quad "reasoning": "brief explanation"\\
\}
\end{quote}
\promptseparator
\promptrole{User Prompt:}

\promptrole{Problem:}

Problem 2. Zoki cycled \(64\%\) of the way, and walked the remaining \(9 \mathrm{~km}\). How many kilometers did Zoki cycle?
\promptseparator
\promptrole{Reference Answer:}

\(16\mathrm{~}\)
\promptseparator
\promptrole{Model Output:}

Let the total distance be \(D\) km.

Zoki cycled \(64\%\) of the way, so the cycling distance is:
\[
0.64D.
\]
The remaining distance is:
\[
D-0.64D=0.36D.
\]
We are told this equals \(9\) km:
\[
0.36D=9.
\]
Solve for \(D\):
\[
D=\frac{9}{0.36}=25.
\]
So the total distance is \(25\) km. Cycling distance:
\[
0.64\times25=16.
\]

Answer: 16 km.
\promptseparator
\promptrole{Evaluation instruction:}
Evaluate if the model's answer is correct. Respond with JSON only.
\promptseparator
\promptrole{Assistant Response:} \answerplaceholder{\textless evaluator JSON output\textgreater}
\end{evalpromptbox}
\end{table*}
\end{document}